\documentclass{article}

 \usepackage[preprint]{neurips_2026}

\usepackage[utf8]{inputenc} % allow utf-8 input
\usepackage[T1]{fontenc}    % use 8-bit T1 fonts
\usepackage{hyperref}       % hyperlinks
\usepackage{url}            % simple URL typesetting
\usepackage{booktabs}       % professional-quality tables
\usepackage{amsfonts}       % blackboard math symbols
\usepackage{nicefrac}       % compact symbols for 1/2, etc.
\usepackage{microtype}      % microtypography
\usepackage{xcolor}         % colors

\usepackage{amsmath}
\usepackage{amssymb}
\usepackage{multirow}
\usepackage{booktabs}
\usepackage{enumitem}

\usepackage{mathtools}
\usepackage{amsthm}
\usepackage{bbm}

\usepackage{bbm}
\usepackage{graphicx}
\usepackage{algorithm} 
\usepackage{algorithmic}
\usepackage{wrapfig}
\usepackage{subcaption}   % for \begin{subfigure}
\usepackage{caption}      % optional: finer caption control
\usepackage{float}        % optional: use [H] if you want strict placement
\usepackage{tikz-cd}
\usepackage{hyperref}
\usepackage{url}
\usepackage{todonotes}
\usepackage{algorithm,algorithmic}
\usepackage{tikz}
\usetikzlibrary{calc,arrows.meta,shapes}  % add calc (and any other libs you need)
\theoremstyle{plain}
\newtheorem{theorem}{Theorem}[section]
\newtheorem{proposition}[theorem]{Proposition}
\newtheorem{lemma}[theorem]{Lemma}

\theoremstyle{definition}
\newtheorem{definition}[theorem]{Definition}

\theoremstyle{remark}
\newtheorem{remark}[theorem]{Remark}

\newcommand{\argmin}{\operatorname*{arg\,min}}

\theoremstyle{definition}

\title{Generative Modeling on Metric Graphs\\ via Neural Optimal Transport}

\author{%
  Alessandro Micheli\thanks{These authors contributed equally to this work.} \\
  Imperial College London\\
  London, UK \\
  \texttt{a.micheli19@imperial.ac.uk} \\
  \And
  Yueqi Cao\footnotemark[1] \\
  KTH Royal Institute of Technology \\
  Stockholm, Sweden \\
  \texttt{yueqic@kth.se} \\
   \And
  Anthea Monod\thanks{These authors jointly supervised this work.}  \\
  Imperial College London\\
  London, UK \\
  \texttt{a.monod@imperial.ac.uk} \\
    \And
  Samir Bhatt\footnotemark[2] \\
  Imperial College London \\
  London, UK \\
  Statens Serum Institut \\
  Copenhagen, Denmark\\ 
  University of Copenhagen \\
  Copenhagen, Denmark\\
  \texttt{s.bhatt@imperial.ac.uk} \\
}

\begin{document}

\maketitle

\begin{abstract}
We introduce, to our knowledge, the first deep generative modeling framework for probability distributions continuously supported on compact metric graphs. Given source and target measures on a metric graph, our method embeds the graph into a smooth ambient space, solves an entropic Kantorovich problem via a neural semidual parameterization, and projects generated samples back onto the original graph. We study two embedded geometries: an extrinsic Euclidean realization and the intrinsic tropical Abel--Jacobi embedding into the Jacobian torus. In both cases, the resulting generator is graph-supported by construction. We prove that, in the joint limit of increasing neural expressivity, the learned generator converges weakly to a valid transport coupling between the original graph measures. Empirically, across a range of geometrically distinct graphs, our method matches or improves upon heuristic transport baselines based on discrete graph OT, while scaling more favorably. Finally, we demonstrate scalability on real-world urban mobility data by training our model on one million Uber pickup locations in Manhattan, New York City.

\end{abstract}

\section{Introduction}
Many scientific and geometric datasets are naturally supported on networks:
road systems, vascular structures, river networks, and
skeletonized shapes all give rise to data distributed over one-dimensional
geometric domains \citep{bachechi2022road,unsalan2012road,reichold2009vascular,dodds2000geometry,reinders2000skeleton}. A natural model for such domains is a \emph{metric graph},
where edges carry lengths and are treated as continuous intervals. This model
retains the geometry of positions along edges while encoding the connectivity
of the underlying network \citep{baker2011metric,burago2001course}. Generative modelling on metric graphs therefore
requires more than generating points in an ambient Euclidean space: samples
should remain supported on the graph and respect the metric geometry of the
domain.

Optimal transport (OT) provides a principled framework for comparing
probability measures and constructing transport mechanisms between them
\citep{Villani2016-ps,Ambrosio2024,ambrosio2012user}. Recent neural OT methods have made these ideas scalable for
high-dimensional continuous domains \citep{universal_neural_OT,rout2022generative,pmlr-v119-makkuva20a,Cuturi2013-ot-2,Altschuler2017-ot}, but they are not directly designed for
probability distributions supported on metric graphs. A common workaround is to
discretize the metric graph and solve OT on the resulting finite graph. While
this makes the problem computationally accessible, it replaces continuous
graph-supported distributions by mesh-dependent approximations and can scale
poorly as the discretization is refined \citep{essid2018quadratically,mcrae2018optimal,arioli2018finite}. These limitations motivate a
continuous, graph-supported neural OT framework that can exploit smooth ambient
structure without losing the geometry of the original metric graph.

We introduce, to our knowledge, the first deep generative modelling framework
for probability distributions continuously supported on compact metric graphs.
Given source and target measures
$\mathbb P_\Gamma,\mathbb Q_\Gamma\in\mathcal P(\Gamma)$, our method embeds
the graph into a smooth ambient space, solves an entropic Kantorovich problem
between the embedded measures using a neural semidual parameterization \citep{cuturi_peyre_semidual_review,pmlr-v202-vacher23a,micheli2026riemannianneuraloptimaltransport}, and
then maps generated samples back to the original graph. The method follows the
pipeline
\[
\Gamma
\xrightarrow{\ \Psi\ }
\mathcal M_\Psi
\xrightarrow{\ \mathrm{Entropic\ Neural\ OT}\ }
\mathcal M_\Psi
\xrightarrow{\ \rho_\Psi^\Gamma\ }
\Gamma.
\]
Thus the learning problem is carried out in a smooth ambient geometry, while
the final generator is supported on the original metric graph by construction.

Our formulation is plan-based rather than map-based. After embedding, the
pushforward measures
$\Psi_\#\mathbb P_\Gamma$ and $\Psi_\#\mathbb Q_\Gamma$ remain supported on the
embedded graph and are typically singular with respect to ambient volume. We do
not mollify these input measures into ambient densities, nor do we rely on
Brenier-type regularity. Instead, we solve entropic OT directly between the
embedded graph-supported measures. The learned Schr\"odinger potential induces
a Gibbs conditional law over target samples. To obtain continuous generative
outputs rather than samples restricted to target atoms, we smooth this
conditional law using the ambient heat kernel. Since heat-smoothed samples need
not lie on the embedded graph, we recover graph-supported samples by nearest
projection onto the embedded graph followed by the inverse embedding.

The choice of embedding determines the geometry in which transport is learned.
We study two complementary choices. The first is an extrinsic Euclidean
realization $\iota:\Gamma\to\mathbb R^m$, which is natural when the graph is
observed with coordinates, as in road networks or embedded skeletons. The
second is the tropical Abel--Jacobi embedding
$\Phi_p:\Gamma\to\mathfrak J(\Gamma)$ into the tropical Jacobian torus
\citep{Mikhalkin2006-ah,baker2011metric,yueqi_computing}. This embedding depends only on the edge lengths and cycle structure
of the metric graph, and therefore provides an intrinsic alternative to a
particular Euclidean drawing. These two choices allow us to compare an
extrinsic ambient geometry with an intrinsic tropical geometry within the same
neural OT framework.

We prove that, in the joint limit of increasing neural expressivity and
vanishing heat scale, the heat-smoothed projection--pullback generator
converges weakly to a valid transport coupling in
$\Pi(\mathbb P_\Gamma,\mathbb Q_\Gamma)$. Consequently, although the method
learns and smooths conditionals in an ambient space, the limiting object is a bona fide coupling between the original graph-supported measures.

Empirically, we evaluate our method on a range of geometrically distinct metric
graphs. Our approach improves upon heuristic transport baselines
based on discrete graph OT, while scaling
comparably than these baselines. We further demonstrate scalability on
real-world urban mobility data by training our model on one million Uber pickup
locations in Manhattan, New York City.

\textbf{Our contributions are as follows.}
We introduce, to our knowledge, the first deep generative modelling framework
for probability distributions continuously supported on compact metric graphs.
Our approach embeds the graph into a smooth ambient space, solves a plan-based
entropic neural OT problem between the embedded graph-supported measures, and
recovers graph-supported samples through a projection--pullback
generator. We develop the framework for two ambient transport geometries: an
extrinsic Euclidean realization and the intrinsic tropical Abel--Jacobi
embedding into the tropical Jacobian torus. We prove that, as neural
expressivity increases, the resulting generator
converges weakly to a valid coupling between the original graph measures.
Empirically, we show that the method improves upon heuristic
baselines based on discrete graph OT while scaling more favorably, and we
demonstrate its scalability to real-wrold data by training it on one million Uber pickup locations in
Manhattan, New York City.

\section{Background}
\label{sec-background}
Our construction combines geometry from metric graphs with neural entropic
optimal transport. We first recall the metric graph formalism and introduce the
two embedding geometries used throughout the paper: an extrinsic Euclidean
realization and the intrinsic tropical Abel--Jacobi embedding into the tropical
Jacobian torus. We then review the semidual formulation of entropic OT, its
neural parameterization, and the heat-smoothed conditional sampling procedure
used by our generator.

\paragraph{Metric Graphs and Their Embeddings.}
A metric graph is a one-dimensional geodesic metric space obtained by assigning
a positive length to each edge of a finite connected graph. Concretely, let
$G=(V,E)$ be a finite connected graph and let
$\ell:E\to\mathbb R_{>0}$ be a length function. The associated metric graph
$\Gamma$ is constructed by gluing a closed interval $[0,\ell(e)]$ for each
$e\in E$ according to the incidence relations of $G$. The distance
$d_\Gamma(x,y)$ between points $x,y\in\Gamma$ is defined as the length of the
shortest continuous path in $\Gamma$ connecting them. We refer to $(G,\ell)$ as
a model of $\Gamma$. An edge is called a bridge if its removal disconnects the
underlying graph. Throughout this paper we restrict to compact metric graphs.

For any metric graph $\Gamma$, we represent its points through an
embedding into a smooth ambient space
\[
    \Psi:\Gamma\longrightarrow \mathcal M_\Psi.
\]
The first natural choice is the Euclidean realization
\[
    \iota:\Gamma\longrightarrow\mathbb R^m,
\]
for example the planar coordinates of a road network \citep{unsalan2012road}. This embedding is efficient for computation, however, the
extrinsic distances in \(\mathbb R^m\) can be small even when the
graph distance is large.

The second embedding is the \emph{tropical Abel--Jacobi embedding}
\citep{yueqi_computing,baker2011metric}, an intrinsic embedding determined by
the edge lengths and cycle structure of the metric graph. It maps $\Gamma$ into
its tropical Jacobian, a flat torus, and is the metric-graph analogue of the
classical Abel--Jacobi map in complex algebraic geometry. We recall only the
ingredients needed for our construction; further details are given in
Appendix~\ref{review-metric-graphs-tropical}.
Computing this distance requires minimizing over the lattice $\Lambda$, i.e.,
solving a closest-vector problem (CVP) in the Jacobian lattice. Exact CVP is
exponential in the lattice dimension in the worst case, so this distance can
become computationally expensive for high-genus graphs; see~\cite[Section 5]{yueqi_computing} for further details.

Fix a basepoint $p\in\Gamma$ and denote by $\Phi_p$ the corresponding tropical
Abel--Jacobi map. On each edge $e\cong[0,\ell(e)]$, the map takes the form
\[
    \Phi_p(t)=[\alpha_e+\beta_e\, t],
\]
for vectors $\alpha_e,\beta_e\in\mathbb R^g$, so the image
$\Phi_p(\Gamma)\subset\mathbb T_\Lambda$ is a finite union of wrapped line segments.

We use the following notation for either embedding geometry:
\[
(\mathcal M_\Psi,d_\Psi,\Psi)
=
\begin{cases}
(\mathbb R^m,\|\cdot-\cdot\|_2,\iota),
& \text{Euclidean embedding},\\[0.3em]
(\mathbb T_\Lambda,d_{\mathbb T_\Lambda,\Xi},\Phi_p),
& \text{tropical Abel--Jacobi embedding}.
\end{cases}
\]
In both cases, the embedded graph $\Psi(\Gamma)$ is the support on which the
input measures live, while $\mathcal M_\Psi$ provides the smooth ambient
geometry used for optimal transport and sampling.
\paragraph{Entropic Riemannian Optimal Transport and Neural Sampling.}

We briefly recall the entropic Riemannian OT framework of
\cite{micheli_entropic}, specialized to the notation used in this paper. Let
$(\mathcal M,\mathfrak g)$ be a complete $p$-dimensional Riemannian manifold
with geodesic distance $d$, and let $\mathcal P(\mathcal M)$ denote the set of
Borel probability measures on $\mathcal M$. We use the quadratic geodesic cost
\[
    c(x,y)=\frac12 d(x,y)^2 .
\]
For $\varepsilon>0$ and
$\mu,\nu\in\mathcal P(\mathcal M)$ with $c\in L^1(\mu\otimes\nu)$, the
entropically regularized OT problem is
\begin{equation}
\label{eq:entropic-ot}
\mathrm{OT}_\varepsilon(\mu,\nu)
=
\inf_{\pi\in\Pi(\mu,\nu)}
\left\{
\int_{\mathcal M\times\mathcal M} c(x,y)\,d\pi(x,y)
+
\varepsilon\,\mathrm{KL}(\pi\mid\mu\otimes\nu)
\right\}.
\end{equation}
Unlike map-valued Monge formulations, the entropic problem naturally returns a
coupling, or equivalently, after disintegration with respect to $\mu$, a
conditional transport law.

The problem \eqref{eq:entropic-ot} admits an equivalent semidual formulation:
the optimization can be reduced to a single target-side dual potential
$g:\mathcal M\to\mathbb R$. Given such a potential, the source-side dual
potential is determined by the soft $c$-transform below. At optimality, this
pair of dual potentials is called the pair of Schr\"odinger potentials. For
$g:\mathcal M\to\mathbb R$, define
\begin{equation}
\label{eq:soft-c-transform}
f_g^\varepsilon(x)
=
-\varepsilon
\log
\int_{\mathcal M}
\exp\!\left(
\frac{g(y)-c(x,y)}{\varepsilon}
\right)
\,d\nu(y).
\end{equation}
The semidual objective is
\begin{equation}
\label{eq:semidual}
\mathcal J_\varepsilon(g)
=
\int_{\mathcal M} g\,d\nu
+
\int_{\mathcal M} f_g^\varepsilon\,d\mu ,
\end{equation}
and, with the supremum taken over admissible potentials,
\begin{equation}
\label{eq:entropic-semidual-equality}
\mathrm{OT}_\varepsilon(\mu,\nu)
=
\sup_g \mathcal J_\varepsilon(g).
\end{equation}

It was shown in \cite{micheli_entropic} that, under suitable expressivity
assumptions, the optimizer of the semidual problem
\eqref{eq:entropic-semidual-equality}, equivalently the optimal target-side
Schr\"odinger potential, can be approximated by restricting the admissible
potentials to a neural function class. Since Schr\"odinger potentials are
defined only up to additive constants, we fix this gauge freedom by centering
the parameterized potential. Given a feature map
$\varphi:\mathcal M\to\mathbb R^n$ and a neural network
$a_\theta:\mathbb R^n\to\mathbb R$, set
\[
h_\theta:=a_\theta\circ\varphi,
\qquad
g_\theta
=
h_\theta-\int h_\theta\,d\nu .
\]
Then $\int g_\theta\,d\nu=0$. The corresponding source-side potential is
\(
    f_\theta^\varepsilon := f_{g_\theta}^\varepsilon,
\)
where $f_{g_\theta}^\varepsilon$ is the soft $c$-transform in
\eqref{eq:soft-c-transform}. The model is trained by maximizing the restricted
semidual objective
\[
    \max_\theta \mathcal J_\varepsilon(g_\theta).
\]

For a trained parameter $\theta$, the learned conditional law is
\begin{equation}
\label{eq:conditional-law}
\pi_{\theta,x}^\varepsilon(dy)
=
\exp\!\left(
\frac{
f_\theta^\varepsilon(x)+g_\theta(y)-c(x,y)
}{\varepsilon}
\right)
\,d\nu(y),
\end{equation}
which is normalized by the definition of $f_\theta^\varepsilon$.Together with the source measure, this gives the learned joint law
\[
    \pi_\theta^\varepsilon(dx,dy)
    =
    \mu(dx)\pi_{\theta,x}^\varepsilon(dy).
\]
Under the expressivity assumptions of \cite{micheli_entropic}, these learned
joint laws approximate the optimal entropic OT plan as the neural function
class becomes increasingly expressive. If
$\nu=\sum_{j=1}^M b_j\delta_{y_j}$ is empirical, then direct sampling from
$\pi_{\theta,x}^\varepsilon$ returns one of the observed target atoms.

To obtain continuous output laws, we smooth these conditionals by the heat
semigroup. Let $p_t$ be the heat kernel on $\mathcal M$ and let $\mathrm{vol}$
be the Riemannian volume measure. For $t>0$, define
\begin{equation}
\label{eq:heat-smoothed-law}
Q_{\theta,t}^\varepsilon(x,dz)
=
\left(
\int_{\mathcal M}
p_t(y,z)\,
\pi_{\theta,x}^\varepsilon(dy)
\right)
\mathrm{vol}(dz).
\end{equation}
Then $Q_{\theta,t}^\varepsilon(x,\cdot)$ is absolutely continuous with respect
to ambient volume and is not restricted to sampled target atoms. In the
empirical case, it is a heat-kernel mixture centered at the atoms
$\{y_j\}_{j=1}^M$, with mixture weights given by the Gibbs conditional.

\section{Graph-Supported Generative Modeling via Neural Optimal Transport}
\label{sec-model}
Let $\Gamma$ be a compact metric graph, and let
$\mathbb P_\Gamma,\mathbb Q_\Gamma\in\mathcal P(\Gamma)$ be source and target
measures, where $\mathcal P(\Gamma)$ denotes the set of Borel probability
measures supported on $\Gamma$. Our goal is to learn a graph-supported
generative mechanism from $\mathbb P_\Gamma$ to $\mathbb Q_\Gamma$. The method
follows the pipeline
\[
\Gamma
\xrightarrow{\ \Psi\ }
\mathcal M_\Psi
\xrightarrow{\ \mathrm{entropic\ neural\ OT}\ }
\mathcal M_\Psi
\xrightarrow{\ \rho_\Psi^\Gamma\ }
\Gamma.
\]
We now describe the three components of this construction; see
Figure~\ref{fig:algorithm} for an illustration of the full pipeline.

\begin{figure}[t!]
    \centering
    \includegraphics[width=\linewidth]{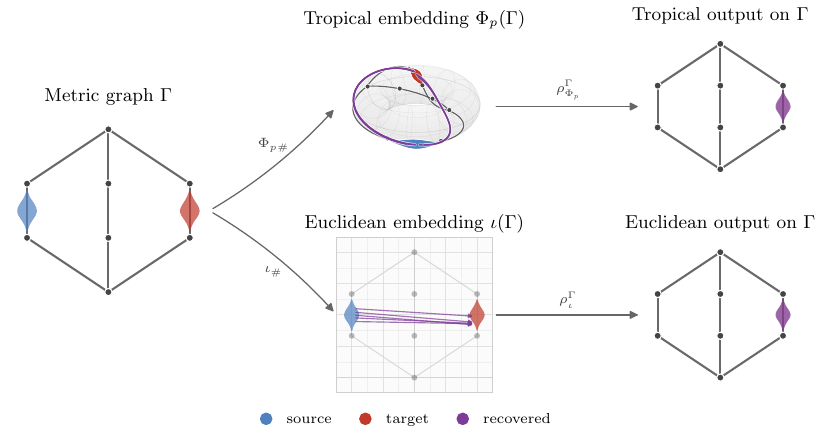}
    \caption{Blue and red denote the source and target measures
$\mathbb P_\Gamma$ and $\mathbb Q_\Gamma$ on the metric graph $\Gamma$.
We compare the two embeddings used in this work: the tropical
Abel--Jacobi map $\Phi_p:\Gamma\to\mathfrak J(\Gamma)$ and the
Euclidean realization $\iota:\Gamma\to\mathbb R^2$.
In each embedded space, a neural entropic OT model generates heat-smoothed
samples, which are mapped back to the graph by the recovery map
$\rho_\Psi^\Gamma$.
Purple denotes the recovered graph-supported output.}
    \label{fig:algorithm}
\end{figure}

\paragraph{Embedding the Graph.}
Let $\Psi:\Gamma\to\mathcal M_\Psi$ be one of the embeddings introduced in
Section~\ref{sec-background}. We use the embedded quadratic cost
\[
c_\Psi(u,v)=\frac12 d_\Psi(u,v)^2,
\qquad u,v\in\mathcal M_\Psi.
\]
For graphs with bridges, the tropical embedding is applied after the bridge
augmentation described in Appendix~\ref{app-bridges}, so that the graph can be
recovered from its embedded image.

For either choice of $\Psi$, we push forward the measures
\begin{equation}
\label{eq:embedded-measures}
\mathbb P_\Psi:=\Psi_\#\mathbb P_\Gamma,
\qquad
\mathbb Q_\Psi:=\Psi_\#\mathbb Q_\Gamma .
\end{equation}
These measures are supported on the embedded graph $\Psi(\Gamma)$ and are
typically singular with respect to ambient volume.

\paragraph{Learning the Embedded Transport Law.}
We apply the neural entropic OT construction of Section~\ref{sec-background}
with
\[
(\mathcal M,\mu,\nu,c)
=
(\mathcal M_\Psi,\mathbb P_\Psi,\mathbb Q_\Psi,c_\Psi).
\]
The corresponding embedded entropic problem is
\begin{equation}
\label{eq:embedded-entropic-ot}
\mathrm{OT}_{\varepsilon}^{\Psi}
=
\inf_{\pi\in\Pi(\mathbb P_\Psi,\mathbb Q_\Psi)}
\left\{
\int c_\Psi(u,v)\,d\pi(u,v)
+
\varepsilon
\mathrm{KL}\!\left(\pi\,\middle|\,\mathbb P_\Psi\otimes\mathbb Q_\Psi\right)
\right\}.
\end{equation}
We parameterize the target-side Schr\"odinger potential by a neural network
$g_\theta:\mathcal M_\Psi\to\mathbb R$. Let
$f_{\theta,\Psi}^{\varepsilon}$ denote the corresponding embedded soft
$c_\Psi$-transform. After training, the learned embedded conditional law is
\[
\pi_{\theta,\Psi,u}^{\varepsilon}(dv)
=
\exp\!\left(
\frac{
f_{\theta,\Psi}^{\varepsilon}(u)+g_\theta(v)-c_\Psi(u,v)
}{\varepsilon}
\right)
\,d\mathbb Q_\Psi(v).
\]

We then smooth this embedded conditional using the heat kernel of
$\mathcal M_\Psi$. Applying \eqref{eq:heat-smoothed-law} with heat kernel
$p_t^\Psi$ gives
\begin{equation}
\label{eq:heat-smoothed-embedded-conditional}
Q_{\theta,\Psi,t}^\varepsilon(u,dz)
=
\left(
\int_{\mathcal M_\Psi}
p_t^\Psi(v,z)
\pi_{\theta,\Psi,u}^{\varepsilon}(dv)
\right)
\mathrm{vol}_\Psi(dz),
\qquad t>0 .
\end{equation}
Samples from $Q_{\theta,\Psi,t}^\varepsilon(u,\cdot)$ are continuous ambient
samples in $\mathcal M_\Psi$. For $t>0$, they need not lie on the embedded
graph $\Psi(\Gamma)$, which motivates the projection--pullback step below.

\paragraph{Recovering Graph-Supported Samples.}
The inverse $\Psi^{-1}$ is defined only on the embedded graph
$\Psi(\Gamma)$. Since heat-smoothed samples live in the ambient space
$\mathcal M_\Psi$, we first project them onto $\Psi(\Gamma)$ and then apply the
inverse embedding. Define the recovery map
\begin{equation}
\label{eq:projection-pullback}
\rho_\Psi^\Gamma
:=
\Psi^{-1}\circ \operatorname{proj}_{\Psi(\Gamma)}
:
\mathcal M_\Psi\to\Gamma,
\end{equation}
where
\[
\operatorname{proj}_{\Psi(\Gamma)}(z)
\in
\argmin_{s\in\Psi(\Gamma)} d_\Psi(z,s)
\]
is the nearest-point projection, with a fixed measurable tie-breaking rule when
the minimizer is not unique. Details of the projection algorithm are given in
Appendix~\ref{app:graph-embedding-projection-inversion}.

The final generator is obtained by pushing the heat-smoothed embedded
conditional through $\rho_\Psi^\Gamma$. Equivalently, for
$x\sim\mathbb P_\Gamma$, we sample
\[
z\sim Q_{\theta,\Psi,t}^{\varepsilon}(\Psi(x),\cdot),
\qquad
y=\rho_\Psi^\Gamma(z).
\]
The induced graph-supported transport law is
\begin{equation}
\label{eq:graph-heat-coupling}
\pi_{\theta,\Gamma,t}^{\varepsilon,\Psi}(dx,dy)
=
\mathbb P_\Gamma(dx)\,
\bigl(\rho_\Psi^\Gamma\bigr)_\#
Q_{\theta,\Psi,t}^{\varepsilon}
\bigl(\Psi(x),\cdot\bigr)(dy),
\end{equation}
which is supported on $\Gamma\times\Gamma$ by construction.

\section{Theoretical Guarantees}
\label{sec:theoretical-guarantees}
We now state the recovery guarantee for the graph-supported generator. The
result shows that, in the joint limit of increasing neural expressivity and
vanishing heat scale, the learned heat-smoothed projection--pullback generator
recovers a valid transport coupling from $\mathbb P_\Gamma$ to
$\mathbb Q_\Gamma$. The
proof is given in Appendix~\ref{app-proof-graph-supported-recovery}.

Fix $\varepsilon>0$ and one of the embeddings
$\Psi:\Gamma\to\mathcal M_\Psi$ from Section~\ref{sec-background}. In both embedding geometries, $\Psi$ is injective onto its image: this is assumed
for the Euclidean realization and ensured for the tropical Abel--Jacobi
embedding by the bridge augmentation in Appendix~\ref{app-bridges}. Hence
$\Psi^{-1}:\Psi(\Gamma)\to\Gamma$ is well-defined. Let
$\pi_{\varepsilon,\Psi}^\star$ be the optimizer of the embedded entropic OT
problem \eqref{eq:embedded-entropic-ot}. Since
$\mathbb P_\Psi=\Psi_\#\mathbb P_\Gamma$ and
$\mathbb Q_\Psi=\Psi_\#\mathbb Q_\Gamma$, the pullback
\begin{equation}
\label{eq:graph-limit-coupling}
\pi_{\varepsilon,\Gamma}^{\star,\Psi}
:=
(\Psi^{-1},\Psi^{-1})_\#\pi_{\varepsilon,\Psi}^\star
\end{equation}
belongs to $\Pi(\mathbb P_\Gamma,\mathbb Q_\Gamma)$.

We now specify the neural potential class. Let
\(
K_{\mathbb Q_\Psi}:=\operatorname{spt}(\mathbb Q_\Psi)\subset\Psi(\Gamma).
\)
Following \cite{micheli_entropic}, we take target-side potentials from a
centered pullback class. Let
$\varphi_\Psi:K_{\mathbb Q_\Psi}\to\mathbb R^n$ be a continuous injective
feature map and let $\mathcal F\subset C(\mathbb R^n,\mathbb R)$ be a function
class. Define
\[
\mathsf C_{\mathbb Q_\Psi}(\varphi_\Psi^*\mathcal F)
:=
\left\{
a\circ\varphi_\Psi
-
\int a\circ\varphi_\Psi\,d\mathbb Q_\Psi
:
a\in\mathcal F
\right\}.
\]
When $\mathcal F$ is dense in $C(\mathbb R^n,\mathbb R)$ under the ucc topology,
the recovery result of \cite{micheli_entropic} gives sequences of neural
potentials in
$\mathsf C_{\mathbb Q_\Psi}(\varphi_\Psi^*\mathcal F)$ whose induced embedded
Gibbs laws recover the embedded entropic coupling. For the two ambient geometries considered in this paper, the heat-kernel
assumption is automatic: $\mathbb R^m$ and the flat torus
$\mathbb T_\Lambda$ are complete and stochastically complete, and their heat
kernels form approximate identities as $t\downarrow0$.

Let $\pi_{m,\Psi}^{\varepsilon}$ denote the embedded Gibbs law induced by a
potential
$g_m\in\mathsf C_{\mathbb Q_\Psi}(\varphi_\Psi^*\mathcal F)$, and let
$\pi_{m,\Gamma,t}^{\varepsilon,\Psi}$ denote the graph-supported
heat-smoothed transport law generated by the procedure in
Section~\ref{sec-model}. Under these assumptions, the recovery result of \cite{micheli_entropic}
ensures the existence of neural potentials in
$\mathsf C_{\mathbb Q_\Psi}(\varphi_\Psi^*\mathcal F)$ whose induced embedded
Gibbs laws converge in total variation to $\pi_{\varepsilon,\Psi}^\star$.

\begin{theorem}[Graph-supported recovery]
\label{thm:graph-supported-recovery}
Assume that $\mathcal F$ is dense in $C(\mathbb R^n,\mathbb R)$ under the ucc
topology and that
$\varphi_\Psi:K_{\mathbb Q_\Psi}\to\mathbb R^n$ is continuous and injective.
Then there exists a sequence of neural potentials
\[
g_m\in\mathsf C_{\mathbb Q_\Psi}(\varphi_\Psi^*\mathcal F)
\]
whose induced embedded Gibbs laws satisfy
\[
\|\pi_{m,\Psi}^{\varepsilon}
-
\pi_{\varepsilon,\Psi}^\star\|_{\mathrm{TV}}
\longrightarrow 0 .
\]
Moreover, there exist sequences $m_k\to\infty$ and $t_k\downarrow0$ such that
\begin{equation}
\label{eq:graph-supported-recovery}
\pi_{m_k,\Gamma,t_k}^{\varepsilon,\Psi}
\rightharpoonup
\pi_{\varepsilon,\Gamma}^{\star,\Psi}
\qquad
(k\to\infty).
\end{equation}
In particular, the learned heat-smoothed generator, followed by the recovery
map $\rho_\Psi^\Gamma$, recovers in the joint limit a valid transport coupling
from $\mathbb P_\Gamma$ to $\mathbb Q_\Gamma$.
\end{theorem}

\section{Experiments}
We evaluate our graph-supported neural OT framework in three complementary
settings. First, we test transport fidelity on controlled intrinsic-geometry
benchmarks, comparing tropical Abel--Jacobi and ambient Euclidean embeddings
against heuristic baselines based on discrete graph OT. Second, we study
computational scaling in time, memory, and inference throughput as graph
resolution and support size increase. Third, we demonstrate scalability on real graph-supported urban mobility data, where observations lie continuously along edge interiors rather than only at vertices. Across experiments, the neural potential is parameterized using feature maps
$\varphi_\Psi$ adapted to the chosen ambient geometry: for the tropical
Abel--Jacobi geometry we use a Gromov feature embedding on the Jacobian torus,
while for the Euclidean geometry we consider both Riemannian log features and
Gromov features. Full implementation details, dataset construction, and
additional experimental results are deferred to
Appendix~\ref{app:implementation-details}. The code will be released publicly
and is currently \textbf{submitted as supplementary material for double-blind
review}.

\subsection{Metric Graph Geometry Benchmarks}
\label{sec:q1}

We first evaluate the proposed neural OT framework on controlled synthetic
benchmarks across four metric graph geometries: a theta graph, a wheel graph, a
rectangular grid, and a small planar road graph. Source and target distributions
are continuous graph-supported measures, typically truncated Gaussian densities
on edge interiors, with the target mode relocated to a geometrically distinct
part of the graph. These benchmarks test transport across homologically distinct
branches, cyclic and radial structure, multiple shortest-path routes, and
irregular road network geometry. We compare against two source-aware
pushforward baselines: a node-interpolation baseline, which solves vertex-level
graph OT and extends the resulting map to edge interiors, and an ambient
pushforward baseline, which computes Euclidean barycentric images before
projecting outputs back to the graph. We evaluate two neural variants: a
tropical Abel--Jacobi model using the Jacobian representation with the tropical
metric, and an ambient Euclidean model using the same training and extraction
pipeline in the Euclidean embedding. The reference target distribution is a
deterministic quadrature approximation of the analytical graph-supported target
measure; full construction details are given in
Appendix~\ref{app:q1-implementation}.

We report intrinsic graph Wasserstein errors $W_1^\Gamma$ and $W_2^\Gamma$
against the deterministic reference measure, together with edgewise density and
CDF $L^1$ errors. Figure~\ref{fig:q1-geometry} reports each neural method's
error relative to the stronger of the two heuristic baselines for each
benchmark and metric; values below one therefore indicate improvement over both
heuristics. Across all metric-graph benchmarks, both neural OT variants
consistently reduce error relative to the best heuristic baseline, with most
improvements falling outside the $\pm 5\%$ effective-tie band. The tropical
Abel--Jacobi and ambient Euclidean neural variants perform comparably overall,
suggesting that the main gain in this experiment comes from learning an
amortized entropic transport rule rather than from a uniform advantage of one
embedding geometry over the other. Full absolute accuracy results for all methods are provided in
Appendix~\ref{app:q1-implementation}.

\begin{figure*}[t]
  \centering
  \includegraphics[width=\textwidth]{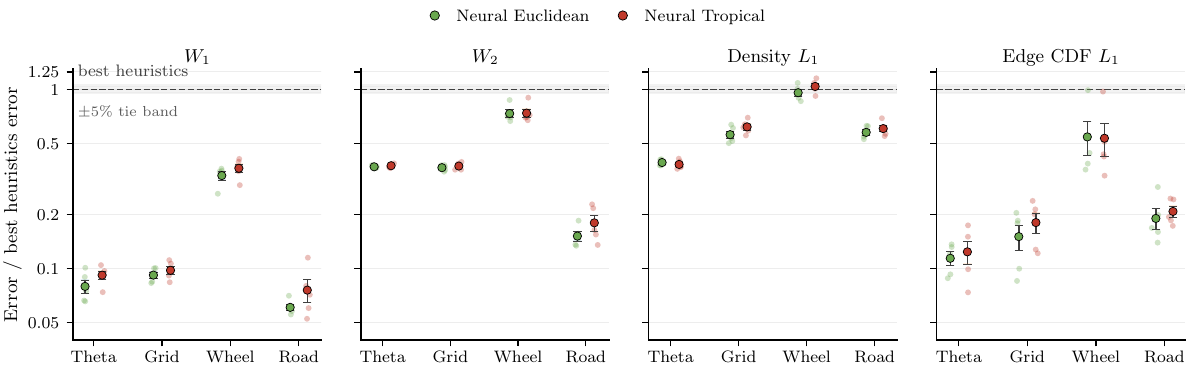}
  \caption{
     Relative error normalized by the best heuristics baseline (better of Ambient / Node-interp pushforward) for each (benchmark, metric). Faint points are individual seeds (n=5); markers are mean ± SEM. Shaded band marks ±5\% of the reference (effective tie). Lower is better; both neural methods consistently improve over the best heuristics.
  }
  \label{fig:q1-geometry}
\end{figure*}

\subsection{Scalability Analysis}
\label{sec:q2}

A central motivation for using amortized neural transport on metric graphs is
computational. Methods based on explicit pairwise graph costs become expensive
as the number of source and target samples increases, since intrinsic
shortest-path distances must be evaluated on increasingly large support pairs.
Our neural formulation instead trains a parametric transport model using
minibatches and applies the learned rule to held-out source samples by batched
inference. To study this scaling behavior, we reuse the synthetic metric-graph
benchmarks from Section~\ref{sec:q1}: theta, wheel, grid, and road graphs with
continuous edge-supported source and target measures. We vary the training
support size
\[
    n_{\mathrm{train}}\in\{128,256,512,1024\},
\]
and evaluate each fitted method on held-out source samples against the same
deterministic quadrature approximation of the analytical target law used in
Section~\ref{sec:q1}. We report intrinsic graph Wasserstein error
$W_1^\Gamma$ as the primary accuracy metric, with additional metrics deferred
to Appendix~\ref{app:impl-q2}. To isolate method cost from evaluation overhead,
we measure fit time and peak GPU memory without running the reference
Sinkhorn solve used for evaluation.

Figure~\ref{fig:q2_4x4} summarizes the sample-size scaling results. The
learned neural OT methods retain their accuracy advantage over the
node-interpolation and ambient pushforward heuristics as the support size
varies, while their memory use is controlled by minibatching and chunked
evaluation rather than by materializing a full pairwise cost matrix. The
tropical Abel--Jacobi and ambient Euclidean neural variants again exhibit
similar accuracy trends, indicating that both provide viable amortized
embeddings for graph-supported transport.

\subsection{Real-World Graph-Supported Generation on Manhattan Uber Pickups}
\label{sec:exp-q3-uber}

\begin{table}[t]
\centering
\caption{%
Million-sample Uber NYC pickup generation on the Manhattan road metric graph.
All neural models are trained on \(10^6\) source samples and \(10^6\) snapped training pickups.
Generation uses \(M=131{,}072\) source queries and a decoder support of \(131{,}072\) target atoms.
Graph \(W_1/W_2\) and density \(L^1\) are evaluated against held-out test pickups; Wasserstein metrics use a fixed \(16{,}384\)-point Sinkhorn subsample for tractability. The \emph{Eval.\ noise floor} row scores two independent subsamples of the held-out test pickups against each other under the identical \(16{,}384\)-point Sinkhorn budget --- the irreducible discrepancy from finite test-set sampling that lower-bounds what any method can attain.}
\label{tab:q3-million}
\begin{tabular}{lcccc}
\toprule
Method & Graph $W_1$ $\downarrow$ & Graph $W_2$ $\downarrow$ & Density $L_1$ $\downarrow$ &  Edge CDF $L_1$ $\downarrow$  \\
\midrule
Eval.\ noise floor & 0.542 {\scriptsize $\pm$ 0.007} & 0.720 {\scriptsize $\pm$ 0.010} & 0.257 {\scriptsize $\pm$ 0.002} & 0.263 {\scriptsize $\pm$ 0.002} \\
\midrule
Ambient pushforward & 7.273 {\scriptsize $\pm$ 0.103} & 4.738 {\scriptsize $\pm$ 0.106} & 1.398 {\scriptsize $\pm$ 0.008} & 1.954 {\scriptsize $\pm$ 0.024}  \\
Node interpolation & 10.048 {\scriptsize $\pm$ 0.138} & 8.029 {\scriptsize $\pm$ 0.252} & 1.717 {\scriptsize $\pm$ 0.008} & 2.598 {\scriptsize $\pm$ 0.019}  \\
Neural Euclidean (Log) & \underline{2.134 {\scriptsize $\pm$ 0.032}} & \underline{1.167 {\scriptsize $\pm$ 0.020}} & \underline{0.538 {\scriptsize $\pm$ 0.002}} & \underline{0.706 {\scriptsize $\pm$ 0.005}}  \\
Neural Euclidean (Gromov) & \textbf{0.662 {\scriptsize $\pm$ 0.010}} & \textbf{0.808 {\scriptsize $\pm$ 0.015}} & \textbf{0.345 {\scriptsize $\pm$ 0.001}} & \textbf{0.385 {\scriptsize $\pm$ 0.004}}  \\
\bottomrule
\end{tabular}
\end{table}

We evaluate our graph-supported neural OT framework on real urban mobility data
derived from the 2014 NYC Uber pickup records. Our goal is to study
large-scale spatial density generation on a metric graph: given a simple
length-uniform source measure on the road network, the model learns to generate
samples from the empirical pickup distribution supported on the same graph.
This experiment is intended as a real-data scalability stress test rather than
a controlled topology-identification benchmark. Unlike the synthetic Q1/Q2
experiments, the target law is not analytically specified; it is the empirical
distribution of observed pickup events after map matching to a Manhattan road
metric graph. The task should therefore be understood as graph-supported
spatial density generation, not temporal forecasting: the model is not
conditioned on date, weather, or time.

The training target measure consists of \(10^6\) Uber pickup locations
sampled from the training days, while the source measure is a length-uniform
distribution on the Manhattan road graph with the same number of samples. Train
and test pickups are drawn from disjoint calendar days to avoid temporal
correlation between fitting and evaluation. The resulting bridge-augmented
metric graph has first Betti number \(g=2416\), making the tropical
Abel--Jacobi model impractical at this scale. We therefore train the neural
Euclidean graph generator, which represents points by their projected road
coordinates and recovers outputs by nearest-polyline projection.

Dense support-dependent OT baselines cannot be run at the full million-sample
scale: a Sinkhorn solve with \(10^6\) source and \(10^6\) target atoms would
require \(10^{12}\) pairwise costs before iterations. Accordingly, the
heuristic baselines are trained on a restricted support of \(16{,}384\) samples.
All methods are evaluated on a separate held-out evaluation set of
\(131{,}072\) graph-supported samples drawn from the test days.
Table~\ref{tab:q3-million} reports the resulting held-out graph
distributional metrics, while Figure~\ref{fig:q3-uber-qualitative} visualizes the recovered spatial density
on the Manhattan road network. The generated
samples recover the dominant pickup structure on the Manhattan graph rather
than placing mass uniformly over roads or off-network, supporting the claim that
the proposed pipeline scales to million-sample empirical target measures while
producing realistic graph-supported samples.

\begin{figure}[t]
    \centering
    \includegraphics[width=0.9\linewidth]{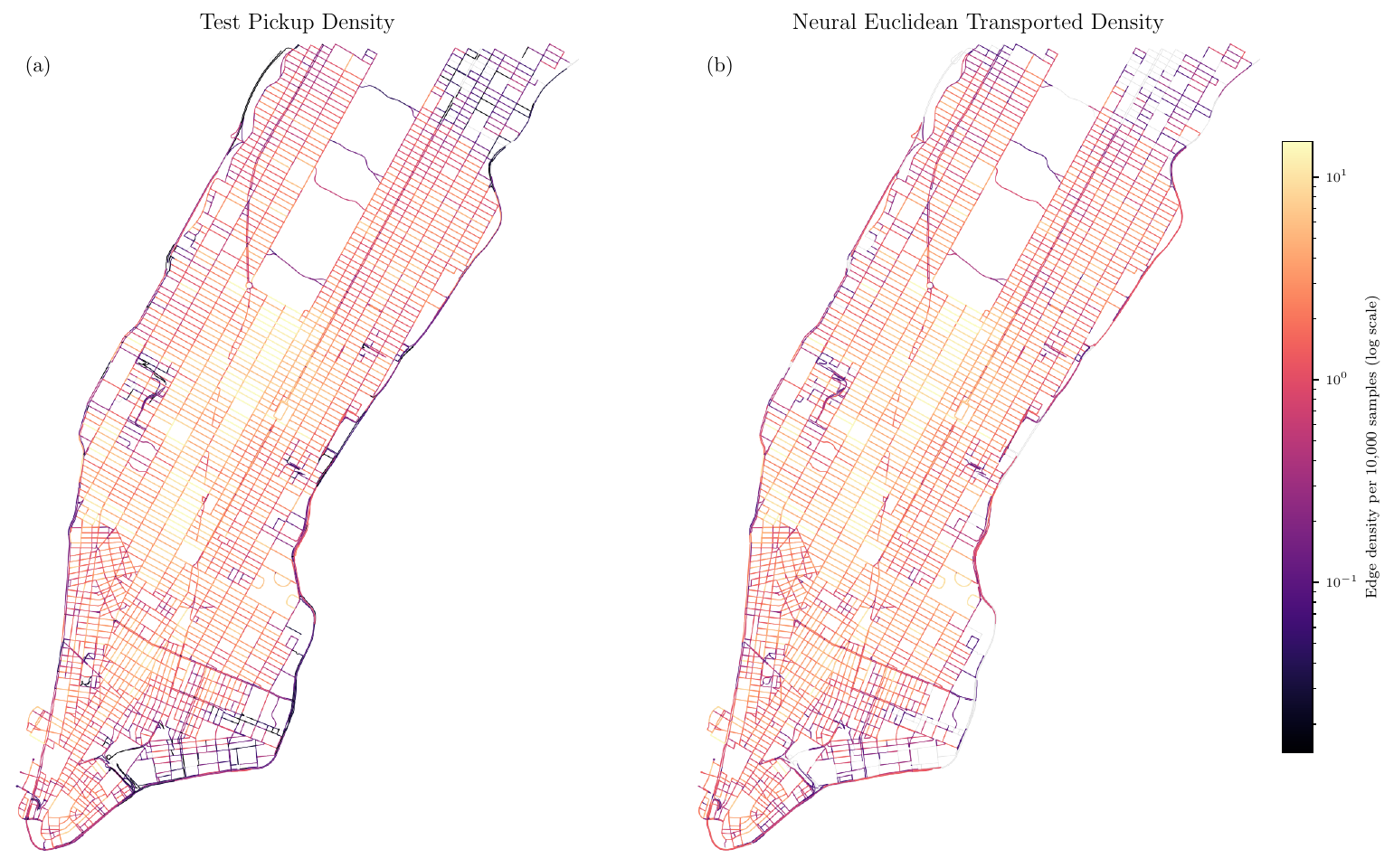}
    \caption{Manhattan road graph at the million-sample scale: per-edge density of (a) held-out test pickups and (b) points transported by Neural Euclidean (Gromov). Quantitative graph metrics are reported in Table~\ref{tab:q3-million}.}
    \label{fig:q3-uber-qualitative}
\end{figure}

\section{Discussion and Limitations}
We introduced an embedded neural OT framework for generative modelling of
probability distributions continuously supported on compact metric graphs. The
method solves a plan-based entropic OT problem after embedding the graph into a
smooth ambient space, then recovers graph-supported samples through a
heat-smoothed projection--pullback map. This avoids replacing the metric graph
by a purely discrete object: learned conditionals are trained against
graph-supported measures, and generated samples lie on edge interiors rather
than only at vertices. The framework applies to both extrinsic Euclidean
realizations and the intrinsic tropical Abel--Jacobi embedding, enabling the
same learning pipeline across transport geometries. Empirically, both neural
variants improve over heuristic baselines based on discrete graph OT on
controlled benchmarks, while the Euclidean model scales to the million-sample
Manhattan experiment where dense graph OT baselines are impractical.

Several limitations remain. The theoretical guarantee is asymptotic: it shows
recovery as neural expressivity increases and the heat scale vanishes, but does
not provide finite-sample or optimization rates, nor explicit rules for choosing
the entropic regularization and heat scale. At finite heat time, the
graph-supported transport law need not have the exact target marginal, since
heat smoothing moves mass off the embedded graph before projection. The
projection--pullback step also relies on nearest-point projection, which can be
non-unique away from the graph. Finally, the real-data experiment should be
viewed as graph-supported spatial density generation rather than forecasting or
causal modelling: the Uber model is not conditioned on time, weather, events, or
other covariates. Extending the framework to conditional or dynamically
evolving graph-supported distributions is a natural direction for future work.

\section*{Acknowledgements}
S.B. acknowledges support from the Novo Nordisk Foundation via The Novo Nordisk Young Investigator Award (NNF20OC0059309). S.B. acknowledges support from The Eric and Wendy Schmidt Fund For Strategic Innovation via the Schmidt Polymath Award (G-22-63345) which also supports A.Micheli. S.B. acknowledges support from the Novo Nordisk Foundation via the Global Pathogen Analysis Platform (GPAP) (NNF26SA0109818). Y.C.~is supported by Digital Futures Postdoctoral Fellowship. A.Monod is supported by the EPSRC AI Hub on Mathematical Foundations of Intelligence: An ``Erlangen Programme'' for AI No.~EP/Y028872/1. 

\section*{Contribution Statements}
Author contributions are reported using the CRediT (Contributor Roles Taxonomy).
\begin{itemize}
     \item \textbf{Alessandro Micheli}: Conceptualization; Methodology; Software; Formal analysis; Supervision; Investigation; Project administration; Visualization; Validation; Writing – original draft; Writing – review \& editing.
     \item \textbf{Yueqi Cao}: Formal analysis; Writing – original draft; Writing – review \& editing.  
     \item \textbf{Anthea Monod}: Writing – original draft; Writing – review \& editing.
     \item \textbf{Samir Bhatt}:
      Software; Funding acquisition; Resources; Visualization; Validation; Writing – original draft; Writing – review \& editing.
 \end{itemize}

\bibliographystyle{plain}
\bibliography{ref} 

\clearpage
\newpage
\appendix
\section{Review on Metric Graphs}
\label{review-metric-graphs-tropical}
\subsection{Geodesic Distance}

A \emph{finite metric graph} is a triple $\Gamma=(V,E,\ell)$ consisting of a finite, connected graph
$(V,E)$ and a length function $\ell:E\to(0,\infty)$, $e\mapsto \ell_e$. We realize $\Gamma$ as the
topological $1$--dimensional CW-complex obtained by replacing each edge $e=\{u,v\}\in E$ with a copy of
the closed interval $[0,\ell_e]$ and identifying its endpoints with the incident vertices $u$ and $v$.

To make the metric explicit, for each edge $e\in E$ fix an isometry (arc-length parametrization)
$\psi_e:[0,\ell_e]\to e\subset\Gamma$. If $\gamma:[a,b]\to\Gamma$ is a continuous path, define its
\emph{length} by
\[
\mathrm{length}(\gamma)
\;:=\;
\sup\Bigl\{\sum_{i=1}^m d_e\bigl(\gamma(t_{i-1}),\gamma(t_i)\bigr):
a=t_0<t_1<\cdots<t_m=b,\ m\in\mathbb N\Bigr\},
\]
where $d_e$ denotes the usual Euclidean distance along the unique edge containing the image of
$\gamma([t_{i-1},t_i])$ (and the sum is taken over a partition refined so that each subpath lies in a
single edge). Equivalently, after refining the partition so that each segment of $\gamma$ is contained
in one edge $e$, write $\gamma|_{[t_{i-1},t_i]}=\psi_e\circ \alpha_i$ with
$\alpha_i:[t_{i-1},t_i]\to[0,\ell_e]$ absolutely continuous; then
\[
\mathrm{length}(\gamma)
=\sum_{i=1}^m\int_{t_{i-1}}^{t_i} |\alpha_i'(t)|\,dt,
\]
and this value does not depend on the chosen parametrizations. The \emph{path metric} $d_\Gamma$ is then defined by
\[
d_\Gamma(x,y)\;:=\;\inf\bigl\{\mathrm{length}(\gamma):\ \gamma\text{ is a continuous path in }\Gamma
\text{ from }x\text{ to }y\bigr\}.
\]

\subsection{Construction of Tropical Embedding}
\label{app-tropical-jacobian}

\paragraph{Tropical Harmonic Forms.}

By fixing a combinatorial model $G$, a tropical harmonic $1$-form can be identified as a collection
$\omega=(\omega_e)_{e\in E}$ of constant functions on the oriented edges, and such that the \emph{Kirchhoff (balancing) condition} holds at every vertex
$v\in V$:
\[
\sum_{e\in E(v)} \sigma(v,e)\,a_e \;=\; 0,
\]
where $E(v)$ denotes the set of edges incident to $v$, and $\sigma(v,e)\in\{+1,-1\}$ equals $+1$ if the
chosen orientation on $e$ points away from $v$ and $-1$ if it points into $v$. The space of harmonic $1$-forms on $\Gamma$ is denoted by $\Omega^1_{\mathrm{harm}}(\Gamma)$. 

For any $\omega\in\Omega^1_{\mathrm{harm}}(\Gamma)$ and any piecewise $C^1$ path
$\gamma$ in $\Gamma$, the path integral $\int_\gamma \omega$ is defined edgewise. Integration over cycles
induces pairing
\[
H_1(\Gamma;\mathbb R)\times \Omega^1_{\mathrm{harm}}(\Gamma)\to\mathbb R,\qquad
([\sigma],\omega)\mapsto \int_\sigma \omega,
\]
and hence the \emph{tropical period map}
\[
\mathrm{Per}:H_1(\Gamma;\mathbb R)\longrightarrow \Omega^1_{\mathrm{harm}}(\Gamma)^\ast,\qquad
\mathrm{Per}([\sigma])(\omega):=\int_\sigma \omega.
\]
The tropical Period map is vector space isomorphism. In particular, it maps the homology group $H_1(\Gamma;\mathbb Z)$
to a full-rank lattice 
in $\Omega^1_{\mathrm{harm}}(\Gamma)^\ast$, denoted by $\Lambda_{\mathrm{per}}:=\mathrm{Per}\bigl(H_1(\Gamma;\mathbb Z)\bigr)$.

\paragraph{Tropical Polarization and Tropical Jacobian.}
The length function on $\Gamma$ induces positive weights on each combinatorial model $G$, which further defines an inner product on the chain space $C_1(G;\mathbb R)$ by
\[
\Big\langle \sum_{e\in E} a_e\, e,\ \sum_{e\in E} b_e\, e\Big\rangle_\ell
:= \sum_{e\in E} a_e b_e\,\ell_e,
\]
which restricts to an inner product on the closed subspace $H_1(G;\mathbb R)\subseteq C_1(G;\mathbb R)$. Since the inner product is invariant to edge subdivision, it is well-defined on $H_1(\Gamma;\mathbb R)$ for the metric graph $\Gamma$. Via the tropical period map
$\mathrm{Per}$, the inner product is pushed forward to $\Omega^1_{\text{harm}}(\Gamma)^*$, which is known as the \emph{tropical polarization}
\begin{equation}
\label{eq:tropical-polarization-dual}
\langle \alpha,\beta\rangle_{\mathrm{trop}}
\;:=\;
\Big\langle \mathrm{Per}^{-1}(\alpha),\, \mathrm{Per}^{-1}(\beta)\Big\rangle_\ell,
\qquad
\alpha,\beta\in \Omega^1_{\mathrm{harm}}(\Gamma)^\ast.
\end{equation}

The \emph{tropical Jacobian} of $\Gamma$ is the quotient torus
\begin{equation}
\label{eq:tropical-jacobian}
\mathfrak{J}(\Gamma)\;:= \Omega^1_{\mathrm{harm}}(\Gamma)^\ast \big/ \Lambda_{\mathrm{per}} \;\cong\; H_1(\Gamma;\mathbb R)\big/ H_1(\Gamma;\mathbb Z),
\end{equation}
equipped with the induced flat metric coming from the polarization
$\langle\cdot,\cdot\rangle_{\mathrm{trop}}$ on $\Omega^1_{\mathrm{harm}}(\Gamma)^\ast$.

We shall consider the following two representations of the tropical Jacobian under bases:

(1) Choosing a $\mathbb Z$-basis $\{\sigma_1,\dots,\sigma_g\}$ of $H_1(\Gamma;\mathbb Z)$ gives 
identifications 

$$
H_1(\Gamma;\mathbb Z)\cong\mathbb Z^g,\qquad H_1(\Gamma;\mathbb R)\cong\mathbb R^g,
$$

under which the inner product $\langle  \cdot,\cdot\rangle_\ell$ is represented by the positive definite matrix $\Sigma=[\langle \sigma_i,\sigma_j\rangle_\ell]_{i,j}$. The flat torus is isometric to the tropical Jacobian by construction;

(2) Let $\Omega^1_{\mathrm{harm}}(\Gamma)_{\mathbb Z}$ be the lattice of tropical harmonic 1-forms with integer coefficients. A $\mathbb Z$-basis $\{\sigma_1,\ldots,\sigma_g\}$ of $H_1(\Gamma;\mathbb Z)$ canonically determines a $\mathbb Z$-basis $\{\omega_1,\ldots,\omega_g\}$ of $\Omega^1_{\mathrm{harm}}(\Gamma)_{\mathbb Z}$, which gives identifications

$$
\Omega^1_{\mathrm{harm}}(\Gamma)^*_{\mathbb Z}\cong \mathbb Z^g,\qquad \Omega^1_{\mathrm{harm}}(\Gamma)^*\cong \mathbb R^g.
$$

Via the tropical period map, the lattice $\Lambda_{\mathrm{per}}$ is given by 

$$
\Lambda_{\mathrm{per}} = \mathrm{span}_{\mathbb Z}\{\mathrm{Per}(\sigma_1),\cdots, \mathrm{Per}(\sigma_g)\},
$$

where each $\mathrm{Per}(\sigma_i)$ is a $g$-dimensional vector in $\mathbb R^g$. Denote the corresponding matrix by $\Lambda$. The tropical polarization $\langle\cdot,\cdot\rangle_{\mathrm{trop}}$ is represented by the matrix $\Xi = \Lambda^T\Sigma\Lambda$.

\paragraph{Tropical Abel--Jacobi Map.} 

Fix a basepoint $p\in\Gamma$. The \emph{tropical Abel--Jacobi map}
$\Phi_p:\Gamma\to\mathfrak{J}(\Gamma)$ sends any point $x\in\Gamma$ to the integral operator over a path from $p$ to $x$. That is,
\[
\Phi_p(x)\;:=\;\Bigl[\ \omega\longmapsto \int_\gamma \omega\ \Bigr]\in
\Omega^1_{\mathrm{harm}}(\Gamma)^\ast / \Lambda_{\mathrm{per}},
\]
where $\gamma$ is any path from $p$ to $x$.  This is well-defined because two such paths differ by a
cycle in $H_1(\Gamma;\mathbb Z)$, hence their associated functionals differ by an element of
$\Lambda_{\mathrm{per}}$. Changing the basepoint translates $\Phi_p$ by a constant element of
$\mathfrak{J}(\Gamma)$.

A key property of the tropical Abel--Jacobi map is that $\Phi_p$ is affine along edges. To make this
explicit, fix an edge $e\in E(\Gamma)$ and an arc-length parametrization
$\varphi_e:[0,\ell_e]\to e$. Under fixed basis we have identification $\mathfrak{J}(\Gamma)\cong \mathbb R^g/\Lambda$. 
Let $\pi:\mathbb R^g\to \mathbb R^g/\Lambda$ be the quotient map. Then the restriction
\[
f:=\Phi_p\circ\varphi_e:[0,\ell_e]\to \mathbb T_\Lambda
\]
admits an affine lift to $\mathbb R^g$: there exist $\widetilde z_e\in\mathbb R^g$ and $u_e\in\mathbb R^g$
such that
\begin{equation}\label{eq:local-lift}
\pi(\widetilde z_e+t u_e)=\Phi_p(\varphi_e(t))
\qquad\text{for all }t\in[0,\ell_e].
\end{equation}
In particular,
\begin{equation}\label{eq:set-def}
\Phi_p(e)=\pi(S_e),
\qquad
S_e:=\{\widetilde z_e+t u_e:\ t\in[0,\ell_e]\}\subset\mathbb R^g,
\end{equation}
so $\Phi_p(e)$ is the image under $\pi$ of a Euclidean line segment.

\clearpage
\newpage

\section{Bridge Elimination by Parallel-Edge Augmentation}
\label{app-bridges}
The results in the main text are stated for \emph{bridgeless} metric graphs because bridges are the only
obstruction to the injectivity of the
tropical Abel--Jacobi map.  In this appendix we
reduce the general case to the bridgeless one by a canonical augmentation: for each bridge we add a
single \emph{parallel} (``virtual'') edge.  Crucially, when the virtual edges are no shorter than the
original bridges, this augmentation de-bridges a metric graph while preserving all pairwise geodesic distances
between points of the original graph.

\subsection{Parallel-Edge Augmentation}

Let $\Gamma$ be a connected metric graph obtained from a finite connected combinatorial graph
$G=(V,E)$ and a length function $\ell:E\to(0,\infty)$ by identifying each edge $e\in E$ with the
interval $[0,\ell(e)]$ and gluing endpoints according to the incidence relations.  For an edge $e$ we
write $e^\circ:=(0,\ell(e))$ for its interior.

\begin{definition}
An edge $b\in E$ is a \emph{bridge} of $\Gamma$ if the topological space $\Gamma\setminus b^\circ$
is disconnected. We denote by $\mathfrak{B}(\Gamma)\subseteq E$ the set of bridges.
\end{definition}

For each bridge $b\in\mathfrak B(\Gamma)$ with endpoints
\(\partial b=\{u_b,v_b\}\), add a new edge \(\bar b\) connecting the same two
endpoints. We call \(\bar b\) a virtual edge and choose its length to be
\begin{equation}
\label{eq:parallel-length}
\ell(\bar b):=\ell(b)+\delta,
\qquad \delta\ge 0.
\end{equation}
In practice, one may take \(\delta=0\), allowing \(\delta>0\) makes the virtual
edge strictly longer while preserving the arguments below. The augmented metric
graph is
\begin{equation}
\label{eq:augmented-graph}
\widetilde\Gamma
:=\Gamma \ \cup \ \bigcup_{b\in\mathfrak{B}(\Gamma)} \bar b,
\end{equation}
equipped with the induced shortest-path metric \(d_{\widetilde\Gamma}\). We
write
\[
\iota:\Gamma\hookrightarrow\widetilde\Gamma
\]
for the canonical inclusion, and
\[
E_{\mathrm{virt}}:=\bigcup_{b\in\mathfrak B(\Gamma)}\bar b
\]
for the union of all virtual edges.

The new edge \(\bar b\) turns the old bridge \(b\) into a two-edge cycle. Thus
neither \(b\) nor \(\bar b\) can disconnect the augmented graph when removed.
Edges that were not bridges before remain non-bridges after adding more paths.
This gives the following basic property.

\begin{lemma}
\label{lem:augmented-bridgeless}
The metric graph $\widetilde\Gamma$ has no bridges.
\end{lemma}
The proof of Lemma~\ref{lem:augmented-bridgeless} is postponed to Appendix~\ref{app-proof-lemma-augmented-bridgeless}.

For any $\delta\ge 0$, the augmentation does not change shortest-path distances
between points of the original graph. In particular, for $\delta=0$, distances are preserved even though shortest
paths in $\widetilde\Gamma$ may fail to be unique because one can swap $b$ and $\bar b$ at equal cost. In fact we can show that $\iota$ is an isometric embedding as stated in the following Lemma.

\begin{lemma}
\label{lem:isometric-inclusion}
Assume $\ell(\bar b)\ge \ell(b)$ for all $b\in\mathfrak{B}(\Gamma)$. Then for all $x,y\in\Gamma$,
\[
d_{\widetilde\Gamma}\big(\iota(x),\iota(y)\big)=d_\Gamma(x,y).
\]
In particular, $\iota$ is an isometric embedding of $(\Gamma,d_\Gamma)$ into $(\widetilde\Gamma,d_{\widetilde\Gamma})$.
\end{lemma}
The proof of Lemma~\ref{lem:isometric-inclusion} is postponed to Appendix~\ref{app-proof-lem:isometric-inclusion}.

We also define a retraction \(r:\widetilde\Gamma\to\Gamma\) as follows:  
Fix arc-length parametrizations \(\psi_b:[0,\ell(b)]\to b\) and
\(\psi_{\bar b}:[0,\ell(\bar b)]\to \bar b\) such that \(\psi_b(0)=\psi_{\bar b}(0)=u_b\) and \(\psi_b(\ell(b))=\psi_{\bar b}(\ell(\bar b))=v_b\). On the original subgraph \(\Gamma\subset\widetilde\Gamma\), we set \(r|_\Gamma=\mathrm{id}_\Gamma\). On each virtual edge \(\bar b\), define \(r\) in arc-length coordinates by
\begin{equation}
\label{eq:retraction}
r\bigl(\psi_{\bar b}(\bar s)\bigr)
\;:=\;
\psi_b\!\left(\frac{\ell(b)}{\ell(\bar b)}\,\bar s\right),
\qquad \bar s\in[0,\ell(\bar b)].
\end{equation}

By construction, \(r\) agrees with the identity on \(\Gamma\), hence it is a left inverse of
inclusion:
\begin{equation}
\label{eq:retraction-property}
r\circ\iota=\mathrm{id}_\Gamma.
\end{equation}
Notice that \(r\) is continuous on each edge and the endpoint conditions ensure that the definitions
match at vertices, so \(r\) is continuous on all of \(\widetilde\Gamma\). Assuming \(\delta\ge 0\), since on each virtual edge \(\bar b\),
the map \(\bar s\mapsto (\ell(b)/\ell(\bar b))\,\bar s\) has slope \(\ell(b)/\ell(\bar b)\le 1\), so
\(r|_{\bar b}\) is \(1\)-Lipschitz with respect to arc-length distance along the edge. Since the shortest-path
metric on \(\widetilde\Gamma\) is obtained by minimizing lengths of piecewise-edge paths, it follows
that \(r\) is globally \(1\)-Lipschitz:
\[
d_\Gamma\bigl(r(x),r(y)\bigr)\;\le\; d_{\widetilde\Gamma}(x,y),
\qquad x,y\in\widetilde\Gamma.
\]

\subsection{Pushforward Measures}

Next we show that augmentation of edges will not affect optimal transport on the original metric graph. Let \(\mathbb P_\Gamma,\mathbb Q_\Gamma\) be Borel probability measures on \(\Gamma\).
We have pushforward measures on \(\widetilde\Gamma\) induced by $\iota$:
\begin{equation}
\label{eq:lifts}
\widetilde{\mathbb P}_\Gamma:=\iota_\#\mathbb P_\Gamma,\qquad
\widetilde{\mathbb Q}_\Gamma:=\iota_\#\mathbb Q_\Gamma.
\end{equation}
Equivalently, for every Borel set \(A\subset\widetilde\Gamma\),
\[
\widetilde{\mathbb P}_\Gamma(A)=\mathbb P_\Gamma\!\bigl(\iota^{-1}(A)\bigr),\qquad
\widetilde{\mathbb Q}_\Gamma(A)=\mathbb Q_\Gamma\!\bigl(\iota^{-1}(A)\bigr).
\]
Since \(\iota(\Gamma)\cap E_{\mathrm{virt}}=\varnothing\), we have \(\iota^{-1}(E_{\mathrm{virt}})=\varnothing\), hence the pushforward measures charge no virtual mass:
\begin{equation}
\label{eq:no-virt-mass}
\widetilde{\mathbb P}_\Gamma(E_{\mathrm{virt}})=\widetilde{\mathbb Q}_\Gamma(E_{\mathrm{virt}})=0.
\end{equation}
Moreover, because \(r\circ\iota=\mathrm{id}_\Gamma\) (cf.~\eqref{eq:retraction-property}), we recover
the original measures after pushing forward by \(r\). Concretely, for every Borel set \(B\subset\Gamma\),
\[
(r_\#\widetilde{\mathbb P}_\Gamma)(B)
=\widetilde{\mathbb P}_\Gamma\!\bigl(r^{-1}(B)\bigr)
=\mathbb P_\Gamma\!\bigl(\iota^{-1}(r^{-1}(B))\bigr)
=\mathbb P_\Gamma\!\bigl((r\circ\iota)^{-1}(B)\bigr)
=\mathbb P_\Gamma(B),
\]
and hence \(r_\#\widetilde{\mathbb P}_\Gamma=\mathbb P_\Gamma\), and similarly \(r_\#\widetilde{\mathbb Q}_\Gamma=\mathbb Q_\Gamma\)..

The same observation applies to transport plans. If \(\widetilde\pi\in\Pi(\widetilde{\mathbb P}_\Gamma,\widetilde{\mathbb Q}_\Gamma)\) is any coupling on \(\widetilde\Gamma\),
define
\begin{equation}
\label{eq:pushback-plan}
\pi:=(r\times r)_\#\widetilde\pi.
\end{equation}
Then \(\pi\) is a coupling of \(\mathbb P_\Gamma\) and \(\mathbb Q_\Gamma\). Indeed, for any Borel set \(A\subset\Gamma\),
\[
\pi(A\times \Gamma)
=\widetilde\pi\bigl((r\times r)^{-1}(A\times\Gamma)\bigr)
=\widetilde\pi\bigl(r^{-1}(A)\times \widetilde\Gamma\bigr)
=\widetilde{\mathbb P}_\Gamma\bigl(r^{-1}(A)\bigr)
=(r_\#\widetilde{\mathbb P}_\Gamma)(A)
=\mathbb P_\Gamma(A),
\]
and similarly \(\pi(\Gamma\times B)=\mathbb Q_\Gamma(B)\) for all Borel \(B\subset\Gamma\). Hence
\(\pi\in\Pi(\mathbb P_\Gamma,\mathbb Q_\Gamma)\).

\begin{remark}
\label{rem:cost-preservation}
Assume $\delta\ge 0$. By Lemma~\ref{lem:isometric-inclusion}, for any cost of the form
$c(x,y)=h(d_\Gamma(x,y))$ with $h:\mathbb{R}_+\to\mathbb{R}_+$, one has
\[
c(x,y)=h(d_\Gamma(x,y))=h\!\left(d_{\widetilde\Gamma}\big(\iota(x),\iota(y)\big)\right).
\]
Hence, when both marginals are supported on $\iota(\Gamma)$, optimization problems posed on $\Gamma$ and on
$\widetilde\Gamma$ with such costs have the same objective value, while $\widetilde\Gamma$ is bridgeless by
Lemma~\ref{lem:augmented-bridgeless}.
\end{remark}

\subsection{Compatibility with Optimal Transport}
\label{app:bridge-ot}

We now explain how the augmentation interacts with the torus formulation in the main text.
Since \(\widetilde\Gamma\) is bridgeless, the tropical Abel--Jacobi map
\[
\widetilde\Phi_p:\widetilde\Gamma\to\mathbb T^{\tilde g}
\]
is defined as in
the bridgeless setting. Here \(\tilde g=g+|\mathfrak{B}(\Gamma)|\) since adding one parallel
edge for each bridge increases the first Betti number by \(1\).

\paragraph{Physical Torus Marginals.} We isolate the part of the torus image coming from the original graph by defining the \emph{physical curve}
\begin{equation}
\label{eq:physical-curve}
\widetilde\Phi_p(\Gamma):=\widetilde\Phi_p(\iota(\Gamma))\subset \widetilde\Phi_p(\widetilde\Gamma),
\end{equation}
and the corresponding torus marginals
\[
\widetilde{\mathbb P}_{\mathfrak{J}}:=(\widetilde\Phi_p)_\#\widetilde{\mathbb P}_\Gamma,
\qquad
\widetilde{\mathbb Q}_{\mathfrak{J}}:=(\widetilde\Phi_p)_\#\widetilde{\mathbb Q}_\Gamma .
\]

\begin{lemma}
\label{lem:torus-support-physical}
The measures $\widetilde{\mathbb P}_{\mathfrak{J}}$ and $\widetilde{\mathbb Q}_{\mathfrak{J}}$ are supported on
$\widetilde\Phi_p(\Gamma)$.
\end{lemma}
The proof of Lemma~\ref{lem:torus-support-physical} is postponed to Appendix~\ref{app-proof-lem:torus-support-physical}.

Consider the entropic OT problem of~\eqref{eq:entropic-ot}. By Lemma~\ref{lem:torus-support-physical}, both marginals are supported on the \emph{physical curve}
\(\widetilde\Phi_p(\Gamma)\subset\mathbb T^{\tilde g}\). Consequently, any coupling between them (and
hence any optimal plan) cannot place mass outside the physical part of the torus.

\begin{lemma}
\label{lem:torus-plan-supported}
Every \(\pi\in\Pi(\widetilde{\mathbb P}_{\mathfrak{J}},\widetilde{\mathbb Q}_{\mathfrak{J}})\) is concentrated on
\(\widetilde\Phi_p(\Gamma)\times \widetilde\Phi_p(\Gamma)\); equivalently,
\[
\pi\!\left(\bigl(\mathbb T^{\tilde g}\times\mathbb T^{\tilde g}\bigr)\setminus
\bigl(\widetilde\Phi_p(\Gamma)\times \widetilde\Phi_p(\Gamma)\bigr)\right)=0.
\]
\end{lemma}
The proof of Lemma~\ref{lem:torus-plan-supported} is postponed to Appendix~\ref{app-proof-lem:torus-plan-supported}.

If a coupling on $\widetilde{\Gamma}$ has null measure on virtual edges, it descends to be a coupling on $\Gamma$. Therefore, we are able to show that parallel-edge augmentation does not change optimal transport plans. 

\begin{proposition}
\label{prop:OT-compat-bridge}
Let $\pi_J^\star$ be an optimal plan for \eqref{eq:entropic-ot}. Then there exists a plan
$\widetilde\pi_\Gamma^\star\in\Pi(\widetilde{\mathbb P}_\Gamma,\widetilde{\mathbb Q}_\Gamma)$ such that
\begin{equation}
\label{eq:pushforward-opt}
(\widetilde\Phi_p\times \widetilde\Phi_p)_\#\widetilde\pi_\Gamma^\star=\pi_J^\star.
\end{equation}
Moreover, $\widetilde\pi_\Gamma^\star$ is concentrated on $\iota(\Gamma)\times\iota(\Gamma)$, and therefore
\[
\widetilde\pi_\Gamma^\star\big(E_{\mathrm{virt}}\times\widetilde\Gamma\big)=
\widetilde\pi_\Gamma^\star\big(\widetilde\Gamma\times E_{\mathrm{virt}}\big)=0.
\]
Identifying $\iota(\Gamma)$ with $\Gamma$, the restriction of $\widetilde\pi_\Gamma^\star$ defines a coupling
$\pi_\Gamma^\star\in\Pi(\mathbb P_\Gamma,\mathbb Q_\Gamma)$ on the original graph.
\end{proposition}
The proof of Proposition~\ref{prop:OT-compat-bridge} is postponed to Appendix~\ref{app-proof-prop:OT-compat-bridge}.

\clearpage
\newpage

\section{Implementation Details}
\label{app:implementation-details}

\subsection{Metric Graph Embedding, Projection, and Inversion}
\label{app:graph-embedding-projection-inversion}

This subsection records the graph-level maps used in the experiments. Here the
embedding refers to the map
\[
    \Psi:\Gamma\to\mathcal M_\Psi
\]
that sends a point of the metric graph into the ambient computational space in
which the OT problem is solved. The neural feature map
\(\varphi_\Psi:\mathcal M_\Psi\to\mathbb R^n\), used to parameterize the
Schr\"odinger potential, is described at the end of the subsection.

\paragraph{Euclidean Embedding and Recovery.}
For the Euclidean geometry, \(\Psi=\iota\) and
\(\mathcal M_\Psi=\mathbb R^2\). The metric graph is embedded by its stored
vertex positions. If \(p_u,p_v\in\mathbb R^2\) are the positions of the
endpoints of an edge \(e=(u,v)\), and if a graph point lies at arclength
\(s\in[0,\ell(e)]\) from \(u\), then
\[
\iota(e,s)
=
(1-\alpha)p_u+\alpha p_v,
\qquad
\alpha=\frac{s}{\ell(e)}.
\]
Vertices are embedded as \(\iota(u)=p_u\). The Euclidean embedded cost is
\[
c_\iota(u,v)=\frac12\|u-v\|_2^2,
\qquad u,v\in\mathbb R^2.
\]
For the synthetic graphs, edges are straight chords, so this is the actual
drawing of the graph. For road graphs with curved OSM geometry, the
implementation uses the corresponding polyline-aware interpolator and
projector.

After Euclidean transport extraction, a sample is an ambient point
\(z\in\mathbb R^2\). The recovery map
\[
\rho_\iota^\Gamma
=
\iota^{-1}\circ\operatorname{proj}_{\iota(\Gamma)}
:
\mathbb R^2\to\Gamma
\]
is implemented by nearest-segment projection. Explicitly,
\[
(e^\star,\alpha^\star)
\in
\arg\min_{e=(u,v),\,\alpha\in[0,1]}
\left\|
z-\bigl((1-\alpha)p_u+\alpha p_v\bigr)
\right\|_2^2 .
\]
For a fixed edge \(e=(u,v)\), the minimizing parameter is
\[
\alpha_e
=
\operatorname{clip}_{[0,1]}
\left(
\frac{\langle z-p_u,p_v-p_u\rangle}{\|p_v-p_u\|_2^2}
\right).
\]
The implementation evaluates this quantity over all edges and selects the edge
with smallest squared residual. The recovered graph point is the endpoint
\(u\) if \(\alpha^\star\) is numerically \(0\), the endpoint \(v\) if
\(\alpha^\star\) is numerically \(1\), and otherwise the edge-interior point
\((e^\star,s^\star)\) with
\[
s^\star=\alpha^\star \ell(e^\star).
\]
Thus Euclidean projection and Euclidean inversion are combined in the recovery
map \(\rho_\iota^\Gamma\).

\paragraph{Tropical Abel--Jacobi Embedding.}
For the tropical geometry, \(\Psi=\Phi_p\) and
\(\mathcal M_\Psi=\mathbb T_\Lambda\), the tropical Jacobian torus. We choose a
spanning tree \(T\), a base vertex \(p\), and the fundamental-cycle basis
determined by the non-tree edges. Let \(g\) be the genus of \(\Gamma\). The
implementation first computes vertex Abel--Jacobi coordinates in a chosen lift
of the Jacobian. In the notation of the code, these raw coordinates are
\[
A
=
C_T L_T Y_T^\top
\in\mathbb R^{g\times |V|},
\]
where \(C_T\) is the reduced cycle-edge incidence matrix on tree edges,
\(L_T\) is the diagonal matrix of tree-edge lengths, and \(Y_T\) is the
basepoint-to-vertex path-edge incidence matrix in the spanning tree. Thus
\(A_v\in\mathbb R^g\) is the raw Abel--Jacobi coordinate of vertex \(v\)
relative to \(p\).

The period matrix is
\[
Q
=
C_T L_T C_T^\top + L_G,
\]
where \(L_G\) is the diagonal matrix of the non-tree edge lengths. In the
notation of the main text, this matrix gives the torus metric matrix
\(\Xi\). Raw coordinates have period lattice \(Q\mathbb Z^g\). In the
implementation, we work in fractional unit-lattice coordinates
\[
\xi = Q^{-1}A \pmod{\mathbb Z^g},
\qquad
\xi\in\mathbb R^g/\mathbb Z^g .
\]
In these coordinates, the tropical torus distance is
\[
d_{\mathbb T_\Lambda,\Xi}(\xi,\eta)^2
=
\min_{n\in\mathbb Z^g}
(\xi-\eta-n)^\top Q(\xi-\eta-n).
\]

For a point on an edge, the Abel--Jacobi coordinate is obtained by affine
interpolation in a lifted edge model. Each oriented edge \(e\) has a lifted
start \(U_e\in\mathbb R^g\) and lifted direction \(d_e\in\mathbb R^g\) in
fractional coordinates. For tree edges,
\[
d_e
=
\xi_{\operatorname{end}(e)}
-
\xi_{\operatorname{start}(e)}.
\]
For the non-tree edge corresponding to cycle \(j\),
\[
d_e
=
\ell(e)\,Q^{-1}e_j .
\]
Hence a graph point at normalized edge coordinate
\(\alpha=s/\ell(e)\) maps to
\[
\Phi_p(e,s)
=
\bigl[U_e+\alpha d_e\bigr]
\in\mathbb R^g/\mathbb Z^g .
\]

\paragraph{Tropical Projection and Inversion.}
After tropical transport extraction, a sample is a torus point
\(\xi\in\mathbb R^g/\mathbb Z^g\). The recovery map
\[
\rho_{\Phi_p}^\Gamma
=
\Phi_p^{-1}\circ\operatorname{proj}_{\Phi_p(\Gamma)}
:
\mathbb T_\Lambda\to\Gamma
\]
is implemented by solving a lifted closest-segment problem:
\[
(e^\star,\alpha^\star,n^\star)
\in
\arg\min_{e,\,\alpha\in[0,1],\,n\in\mathbb Z^g}
\left\|
\xi-U_e-n-\alpha d_e
\right\|_Q^2,
\qquad
\|r\|_Q^2:=r^\top Qr .
\]
For fixed \((e,n)\), the optimal edge parameter has a closed form,
\[
\alpha_{e,n}
=
\operatorname{clip}_{[0,1]}
\left(
\frac{(\xi-U_e-n)^\top Qd_e}{d_e^\top Qd_e}
\right),
\]
so the only non-trivial degree of freedom is the integer shift \(n\).

\emph{Certified-exact bounded-box enumeration.}  The default
implementation enumerates \(n\) over a bounded integer box
\(n\in[-r,r]^g\) and, for each \((e,n)\), evaluates the closed-form
clipped \(\alpha_{e,n}\) and the resulting residual; the global
\((e^\star,n^\star)\) is the argmin over edges and shifts. This
enumeration is \emph{certified exact} once \(r\) reaches the
spectral certified radius
\[
   r_{\mathrm{cert}}(Q)
   = \left\lceil
       \tfrac12\sqrt{g\,\lambda_{\max}(Q)/\lambda_{\min}(Q)}-\tfrac12
     \right\rceil ,
\]
because any candidate \(n'\) outside the box satisfies
\(\|\xi-U_e-n'-\alpha d_e\|_Q\ge \|n'\|_{Q^{-1}}^{-1}/2 > r_Q^\star\),
where \(r_Q^\star\) bounds the interior-box residual.

\emph{High-genus fallback.}  When a benchmark's effective genus exceeds
the cap of 7, the certified box becomes
prohibitively large (\(|n|=O((2r+1)^g)\)) and the implementation
auto-switches to a Babai nearest-plane solver on an LLL-reduced basis.

\emph{Decoder output.}  The same closest-segment problem implements
both projection onto \(\Phi_p(\Gamma)\) and inversion back to the
metric graph. Given \(\xi\), the decoder returns the edge
\(e^\star\), edge coordinate \(\alpha^\star\), lattice shift
\(n^\star\), and residual
\[
r^\star
=
\left\|
\xi-U_{e^\star}-n^\star-\alpha^\star d_{e^\star}
\right\|_Q .
\]
If \(r^\star\) is below numerical tolerance, \(\xi\) is treated as
lying on the Abel--Jacobi image of the graph; otherwise the same
minimizer gives the nearest point on \(\Phi_p(\Gamma)\). The decoded
graph point is
\[
(e^\star,s^\star),
\qquad
s^\star=\alpha^\star\ell(e^\star),
\]
with endpoint snapping when \(\alpha^\star\) is numerically \(0\) or
\(1\).
\paragraph{Neural Feature Maps.}
The preceding paragraphs describe the graph-to-ambient maps \(\Psi\) and the
corresponding recovery maps \(\rho_\Psi^\Gamma\). The neural potential is
parameterized using a feature map
\[
\varphi_\Psi:\mathcal M_\Psi\to\mathbb R^n .
\]
For the tropical Abel--Jacobi geometry, we use a Gromov distance feature map on
the Jacobian torus: the features of a point are its geodesic distances to a
fixed set of landmark points, typically embedded vertices, with source samples
used to fill any remaining landmark slots. For the Euclidean geometry, we
consider both Riemannian log features and Gromov distance features. In
\(\mathbb R^2\), the log feature at a reference point \(\bar x\) is simply
\(x\mapsto x-\bar x\).
\subsection{Implementation Details for Metric Graph Geometry Benchmarks}
\label{app:q1-implementation}

This subsection gives the implementation details for the synthetic
metric-graph geometry experiment in Section~\ref{sec:q1}. Since the goal of this experiment
is to evaluate \emph{graph-supported sample generation} and the induced
\emph{target-measure recovery} on metric graphs, every method is evaluated
through the graph-supported measure \(\widehat\nu\) induced by its transported
or generated samples. The reference object is the true continuous target law
\(\nu=\mathbb Q_\Gamma\), approximated by a deterministic graph-supported
quadrature rule shared across all methods and seeds.

\paragraph{Experiment Protocol.}
\label{app:q1-protocol}

Each benchmark consists of a metric graph \(\Gamma=(V,E,\ell)\), a source
probability measure \(\mu\), and a target probability measure \(\nu\), both
supported on \(\Gamma\). The target measure is constructed analytically by
moving the source mode to a different part of the graph, so that evaluation
can be performed against a deterministic quadrature approximation of the true
target law rather than against a finite Monte Carlo sample.

For each benchmark and seed, we draw three independent samples:
\[
    X_{\mathrm{train}} \sim \mu^{\otimes n_{\mathrm{train}}},
    \qquad
    Y_{\mathrm{train}} \sim \nu^{\otimes n_{\mathrm{train}}},
    \qquad
    X_{\mathrm{eval}} \sim \mu^{\otimes n_{\mathrm{test}}}.
\]
The source-aware methods fit a transport rule using
\((X_{\mathrm{train}},Y_{\mathrm{train}})\), and are then evaluated by applying
the fitted rule to the held-out source batch \(X_{\mathrm{eval}}\). This
train/eval split prevents source-aware methods from being evaluated on the
same source samples used to fit their transport plan. 

The full experimental configuration uses
\[
    n_{\mathrm{train}}=2048,\qquad
    n_{\mathrm{test}}=512,\qquad
    n_{\mathrm{ref}}=10{,}000,
\]
with five seeds \(0,1,2,3,4\). The smoke configuration uses
\(n_{\mathrm{train}}=128\), \(n_{\mathrm{test}}=64\), \(n_{\mathrm{ref}}=512\),
and a single seed; it additionally shrinks the neural-training schedule
(\(n_{\mathrm{steps}}=30\), batch size \(64\)) and the evaluation
discretisation (\(8\) bins per edge, \(100\) Sinkhorn iterations) so that the
end-to-end pipeline runs in seconds.  Unless otherwise stated, the full
configuration uses \(128\) bins per edge, and graph Wasserstein metrics use
entropic Sinkhorn with regularisation \(10^{-3}\) and \(400\) iterations.

\paragraph{Graph-Supported Probability Measures.}
\label{app:q1-graph-measures}

Each synthetic source or target law is represented as a probability measure
supported on a compact metric graph \(\Gamma=(V,E,\ell)\). Every edge
\(e\in E\) is identified with an interval \([0,\ell_e]\), with canonical
arc-length parametrisation
\[
    \iota:[0,\ell_e]\longrightarrow \Gamma .
\]
A graph point is represented either by a vertex \(v\in V\), or by an
edge--arc-length pair \((e,s)\) with \(s\in[0,\ell_e]\).

The measures used in the first experiment are mixtures of edge-continuous components and,
optionally, vertex atoms:
\[
    P_\Gamma
    =
    \sum_{r=1}^{R_{\mathrm{edge}}}
    \alpha_r\,(\iota_{e_r})_{\#}\rho_r
    +
    \sum_{u=1}^{R_{\mathrm{vert}}}
    \beta_u\,\delta_{v_u},
    \qquad
    \alpha_r,\beta_u\ge 0,
    \qquad
    \sum_r\alpha_r+\sum_u\beta_u=1 .
\]
Here \(\rho_r\) is a one-dimensional density on the edge interval
\([0,\ell_{e_r}]\). In the benchmarks, the source and target components
are typically truncated Gaussian edge densities,
\[
    \rho_{e,\mu,\sigma}(s)
    =
    \frac{
        \exp\!\left(-\frac{(s-\mu)^2}{2\sigma^2}\right)
        \mathbbm {1}_{\{0\le s\le \ell_e\}}
    }{
        \int_0^{\ell_e}
        \exp\!\left(-\frac{(t-\mu)^2}{2\sigma^2}\right)\,dt
    }.
\]
The implementation also supports uniform edge densities, triangular edge
densities, and vertex atoms. The target law \(\nu=\mathbb Q_\Gamma\) is
obtained from the source law by relocating the edge-supported mode to a
different graph region, for example across theta branches, wheel spokes,
grid corners, or road-network hotspot edges.

Graph points are represented in code either as a vertex \(u\in V\), or as
\(((u,v),s)\), meaning the point at arc-length \(s\in[0,\ell_{uv}]\) from
\(u\) along edge \((u,v)\). Reverse orientations are handled consistently:
a component stored on \((v,u)\) is evaluated on \((u,v)\) by the coordinate
transformation \(s\mapsto \ell_{uv}-s\).

\paragraph{Synthetic Benchmarks.}
\label{app:q1-benchmarks}

The default experiment sweep uses four graph families. All edge weights are metric
lengths. For synthetic graphs embedded in \(\mathbb R^2\), edge weights are
the Euclidean lengths of the corresponding line segments.

\emph{Theta graph.}
The theta benchmark uses \(\Theta(k=3)\) with two interior vertices per
branch. The source is a truncated Gaussian on an internal edge of branch
\(0\), and the target is the corresponding truncated Gaussian on an internal
edge of branch \(2\). This benchmark tests transport across homologically
distinct branches.

\emph{Wheel graph.}
The wheel benchmark uses a wheel graph with six vertices, one hub and five
rim vertices. The source is a truncated Gaussian on one spoke, and the
target is a truncated Gaussian on a non-adjacent spoke. This graph has
multiple cycles and tests whether methods respect the radial/cyclic
structure.

\emph{Grid graph.}
The grid benchmark uses a \(3\times 3\) rectangular grid. The source is
placed on an edge near the bottom-left corner and the target on an edge near
the top-right corner. This benchmark tests transport across a graph with
several homotopy-equivalent routes.

\emph{Road graph.}
The road benchmark is a small random planar graph generated from a Delaunay
triangulation with \(10\) nodes, then pruned while preserving connectivity.
The source is placed on an edge near one corner of the embedding and the
target on an edge near the diagonal corner. After bridge augmentation the
effective Jacobian-torus dimension can reach \(g\approx 8\); the
implementation handles this via a high-genus fallback rule described
below, so that the experimental protocol works on Road without manual configuration.

\paragraph{Deterministic Reference Measure.}
\label{app:q1-reference}

For evaluation, we do not use a Monte Carlo draw from the target as the
reference. Instead, for each benchmark we build a deterministic quadrature
approximation of the analytical target law \(\nu\). For every target edge
component with normalised mass \(\alpha_r\), we allocate a number of
midpoint quadrature nodes proportional to \(\alpha_r\), evaluate the
one-dimensional density on that grid, and renormalise the quadrature
weights on that edge. The reference measure has the form
\[
    \nu_{\mathrm{ref}}
    =
    \sum_{k=1}^{K_{\mathrm{ref}}}
    w_k \delta_{z_k},
    \qquad
    \sum_k w_k = 1.
\]
Vertex atoms, when present, are included exactly with their normalised
masses. This reference is built once per benchmark and reused for every
seed and method, so metric variation reflects method and training
variability rather than reference-sampling noise.

\paragraph{Methods.}
\label{app:q1-methods}

The main comparison includes three source-aware baselines, the
Neural Tropical method and two Neural Euclidean baselines (Log and Gromov embedding, respectively).

\emph{Node-interpolation pushforward.}
This baseline first snaps source and target samples to their nearest
vertices and solves vertex-level entropic OT with intrinsic graph cost
\[
    c_\Gamma(u,v) = \frac12 d_\Gamma(u,v)^2.
\]
The resulting vertex-level barycentric images are extended to edge-interior
source points by linear interpolation along the source edge. A second
optional ambient Euclidean Sinkhorn assignment maps the interpolated
positions toward the full target sample support before projecting back to
the graph.

\emph{Ambient Euclidean pushforward.}
This baseline embeds source and target graph points in the ambient
\(\mathbb R^2\) plane, solves entropic OT with squared Euclidean cost
\[
    c_{\mathbb R^2}(x,y)=\frac12\|x-y\|^2,
\]
forms barycentric images, and projects the resulting ambient points back to
the nearest point on the graph.

\emph{Neural tropical method.}
This method maps graph points into the tropical Jacobian torus via
the Abel--Jacobi map
\[
    \Phi:\Gamma\longrightarrow \operatorname{Jac}(\Gamma).
\]
After choosing a cycle basis, the Jacobian is represented as a flat torus
\[
    \operatorname{Jac}(\Gamma)\cong \mathbb R^g/\mathbb Z^g
\]
equipped with the tropical metric matrix \(Q\). We train an entropic
semi-dual potential on the torus using the cost
\[
    c_{\operatorname{Jac}}(x,y)
    =
    \frac12
    \min_{n\in\mathbb Z^g}
    (x-y-n)^\top Q (x-y-n).
\]
Given a held-out source point \(x\), the learned potential defines a Gibbs
conditional distribution \(P_x(\,\cdot\,)\) over the lifted target support.
The experiment uses a \emph{heat-smoothed sampling} decoder: rather than reading off
the conditional argmax, we draw a torus point from
\[
    q_t(z) \;=\; \sum_{j} P_x(j)\,p_t(y_j,z),
\]
where \(p_t\) is the torus heat kernel and the heat time scales linearly
with the entropic temperature, \(t = \alpha\,\varepsilon\) with the small
constant \(\alpha = 3\!\times\!10^{-3}\). The Varadhan-matching choice
\(\alpha = \tfrac12\) over-smooths catastrophically on the small synthetic
supports used in the first experiment because the Brownian standard deviation
\(\sqrt{2t} = \sqrt{\alpha\,\varepsilon}\) becomes comparable to the
inter-atom spacing on the lifted target support and the sample crosses to
neighbouring edges; the small-\(\alpha\) regime keeps the smoothing within
the resident edge so that final-projection accuracy is preserved.

The torus draw is then projected back to the graph image by solving the
inverse Abel--Jacobi projection problem discussed in Appendix~\ref{app:graph-embedding-projection-inversion}.

\emph{Neural Euclidean methods.}
The Euclidean ablation uses the same neural entropic OT training loop,
learning rate, batch size, adaptive regularisation rule, and heat-smoothed
sampling decoder as the neural tropical method. The only change is the
ambient geometry: source and target samples are represented in
\(\mathbb R^2\), the cost is squared Euclidean distance, and the final
output is projected back to the graph in the ambient plane. This isolates
the effect of the Abel--Jacobi / tropical-Jacobian inductive bias.

\paragraph{Adaptive Entropic Regularisation.}
\label{app:q1-adaptive-epsilon}

All source-aware methods use the same adaptive entropic regularisation
convention. For each method, we compute a cost matrix on a random
subsample of at most \(256\) source points and \(256\) target points,
using that method's own geometry and cost. The entropic temperature is
then set to
\[
    \varepsilon
    =
    \lambda \cdot \operatorname{median}(C),
    \qquad
    \lambda = 0.01.
\]
This prevents a fixed absolute value of \(\varepsilon\) from favouring one
geometry simply because its cost scale is smaller. If adaptive
regularisation is disabled, the implementation falls back to the absolute
temperatures specified in the configuration.

\paragraph{Evaluation Metrics.}
\label{app:q1-metrics}

We report four metrics comparing a predicted graph-supported measure
\(\hat\nu\) with the deterministic reference \(\nu_{\mathrm{ref}}\).

\emph{Graph Wasserstein distances.}
The primary native-form metrics are entropic approximations of intrinsic
graph Wasserstein distances:
\[
    W_1^\Gamma(\hat\nu,\nu_{\mathrm{ref}})
    \quad\text{and}\quad
    W_2^\Gamma(\hat\nu,\nu_{\mathrm{ref}}).
\]
Both use the exact graph geodesic distance \(d_\Gamma\) between graph
points. For edge-interior points, the distance is computed as the minimum
over the four endpoint-routed paths, with the direct same-edge arc-length
included as an additional candidate. Thus alternate routes through the
graph are correctly accounted for even when two points lie on the same
edge.

For \(W_1\), the Sinkhorn cost matrix is \(C_{ij}=d_\Gamma(x_i,y_j)\). For
\(W_2\), the cost matrix is \(C_{ij}=d_\Gamma(x_i,y_j)^2\), and the square
root of the optimal transport cost is returned.

\emph{Edgewise density \(L^1\).}
To compare density-like behaviour, both measures are binned onto a common
edge-bin grid. Let \(\rho_{\hat\nu}\) and \(\rho_{\mathrm{ref}}\) denote
the piecewise-constant densities induced on this grid. We compute
\[
    \|\rho_{\hat\nu}-\rho_{\mathrm{ref}}\|_{L^1(\Gamma)}
    =
    \sum_{e\in E}
    \int_e
    |\rho_{\hat\nu}(s)-\rho_{\mathrm{ref}}(s)|\,ds.
\]
Vertex mass is split evenly across incident endpoint bins.

\emph{Edgewise CDF \(L^1\).}
For each edge \(e\), we compute the cumulative mass functions induced by
the common edge bins and integrate the absolute CDF error along the edge:
\[
    \sum_{e\in E}
    \int_0^{\ell_e}
    |F_{\hat\nu,e}(s)-F_{\mathrm{ref},e}(s)|\,ds.
\]
This metric is less sensitive than density \(L^1\) to small within-edge
shifts while still detecting whether mass is placed on the correct edge
and in the correct region.

\paragraph{Tropical Abel--Jacobi Implementation.}
\label{app:q1-aj-implementation}

The Abel--Jacobi implementation precomputes the period matrix \(Q\), its
inverse, vertex Abel--Jacobi coordinates, and lifted edge starts /
directions. For tree edges, the lifted edge direction is the difference
of the lifted endpoint coordinates. For non-tree edges, the lifted
direction is determined by the corresponding cycle-basis coordinate and
the metric edge length. The implementation orients each non-tree edge so
that the forward map and inverse projection are consistent.

The forward map is implemented by locating each ambient graph point on its
supporting edge, interpolating the Abel--Jacobi coordinate along that edge,
and converting to fractional torus coordinates. The inverse / projection
map solves the edge / lattice least-squares problem described above.

Graphs with bridges can be augmented by adding a virtual parallel
half-edge path through a midpoint, creating a small cycle for each
bridge. This gives bridge-supported mass a representation in the
Jacobian torus. After tropical projection on the augmented graph,
predicted positions are snapped back to the original graph for
evaluation.

\paragraph{Numerical and Implementation Notes.}
\label{app:q1-numerical-notes}

All Sinkhorn solves use the log-domain implementation from OTT-JAX for
numerical stability. Neural methods use JAX/JIT-compiled kernels for the
Abel--Jacobi forward map, torus CVP computations, and inverse projection.
Evaluation and projection are chunked to control GPU memory: the default roster uses an eval-time transport chunk of \(64\), a one-shot
\(\psi(y)\) precomputation chunk of \(256\), and an Abel--Jacobi
projection chunk of \(16\); these chunks affect only memory scheduling,
not the mathematical estimator. In the default roster, the neural
methods use
\[
    \texttt{batch\_size}=256,\qquad
    \texttt{n\_steps}=3000,\qquad
    \texttt{learning\_rate}=10^{-2},
\]
with cosine learning-rate decay enabled.  Checkpointing can be enabled to
save one neural potential per benchmark and seed.

\subsection{Implementation Details for Scaling Analysis}
\label{app:impl-q2}

This experiment reuses the synthetic metric-graph geometry benchmarks and method
implementations from the first synthetic experiment, but changes the experimental question from
``which method recovers the target measure best?'' to ``how do accuracy,
runtime, and memory scale as the number of training samples?''  The code therefore separates \emph{accuracy
evaluation} from \emph{scaling measurement}: accuracy runs compute
graph-level discrepancies against a deterministic reference target, while
scaling runs measure only the cost of fitting and applying the method,
deliberately excluding the expensive reference-evaluation Sinkhorn solve.

\paragraph{Benchmarks and Sample-Size Sweep.}
The main sample-size sweep uses the same four benchmark instances as the
the first experiment sweep (theta graph, grid graph, wheel graph and road graph)
with first Betti numbers \(g=2,\,4,\,5,\,\ge 7\) respectively (Road graph's
effective genus depends on bridge augmentation).  The training sizes are
\[
n_{\mathrm{train}}\in\{128,256,512,1024\},
\]
and each cell is repeated over seeds
\[
\{0,1,2\}.
\]
For each seed, we draw independent source-training, target-training, and
held-out source-evaluation samples.  Methods are fit on
\((\{x_i\}_{i=1}^{n_{\mathrm{train}}},\{y_j\}_{j=1}^{n_{\mathrm{train}}})\)
and then evaluated by transporting a held-out source cloud of size
\(n_{\mathrm{test}}=512\), inherited from the first experiment configuration.  

\paragraph{Methods.}
The roster in this experiment contains the following source-aware methods: Ambient pushforward, Node Interpolation, NeuralEuclidean (with Log embedding), NeuralEuclidean(with Gromov embedding), NeuralTropical.
The first two are deterministic pushforward baselines.   \textsc{NeuralEuclidean}
appears as two ablation variants that share architecture and training
schedule but differ in the input embedding: \textsc{NeuralEuclidean(Log)}
uses a logarithmic embedding in \(\mathbb R^2\) (centering at a base point)
with LayerNorm disabled, while \textsc{NeuralEuclidean(Gromov)} uses a
graph-distance landmark embedding with LayerNorm enabled.  Both are
ambient-Euclidean ablations of the Jacobian-torus method
\textsc{NeuralTropical}. 

\paragraph{Adaptive Regularisation and Chunking.}
For methods using entropic OT, the regularisation scale is selected as
\[
\varepsilon = 0.01\,\mathrm{median}(C),
\]
where the median is computed on a fixed \(256\times256\) random subsample
of the relevant cost matrix for that method.  This keeps the temperature
comparable across graph families and prevents a full
\(n_{\mathrm{train}}^2\) median-cost matrix from being materialised only
to tune \(\varepsilon\).

The neural methods use chunked evaluation to keep peak memory bounded.  In
the default  configuration, the transported source-evaluation batch
is split into chunks of size \(64\), the one-shot \(\psi(y)\) precomputation
is split into chunks of size \(256\), and Abel--Jacobi projection is split
into chunks of size \(16\).  These chunks affect only memory scheduling;
they do not change the mathematical estimator.

\paragraph{Tropical CVP and Projection.}
The tropical methods use the same Abel--Jacobi data structures as Q1.  The
default CVP method is exact whenever the certified lattice-search radius
is within the configured budget.  When a graph's effective genus
exceeds the cap, the implementation falls back to a scalable Babai
nearest-plane solver on an LLL-reduced basis; this avoids silently
treating a fixed small search box as exact.  Approximate CVP variants
based on LLL or BKZ can also be selected through the configuration, but
the default Q2 sweep uses the same production-safe settings as Q1.  In
practice, this means Theta (\(g=2\)), Grid (\(g=4\)), and Wheel (\(g=5\))
run with exact CVP, while Road (effective \(g\ge 7\) after bridge
augmentation) auto-falls back to Babai nearest plane solver.

\paragraph{Accuracy Sweep.}
The accuracy runner computes graph-level error as a function of
\(n_{\mathrm{train}}\).  For each benchmark, a deterministic edgewise
quadrature reference is built once from the analytical target measure
\(\mathbb Q_\Gamma\) and shared across methods and seeds:
\[
N_{\mathrm{ref}}=10{,}000
\]
quadrature nodes.  The vertex-distance matrix \(D_{vv}\) is precomputed
once per benchmark and reused by every \((\mathrm{method},\mathrm{seed})\)
cell.  Each predicted graph-supported measure \(\widehat\nu\) is compared
to this reference using the same metrics as Q1:
\[
W_1(\widehat\nu,\mathbb Q_\Gamma),\qquad
W_2(\widehat\nu,\mathbb Q_\Gamma),\qquad
\|\widehat\rho-\rho\|_{L^1(\Gamma)},
\]
and the edgewise CDF \(L^1\) diagnostic.  The reference is
seed-independent, so all \((\mathrm{method},\mathrm{seed})\) cells are
scored against the same target object and metric variance is attributable
to method/training noise alone.

\paragraph{Scaling Sweep and Memory Measurement.}
In the scaling sweep, we measure method fit time and peak GPU memory without calling the reference-evaluation routine. This is intentional: evaluating $W_1$ or $W_2$ against the $N_{\mathrm{ref}} = 10{,}000$ quadrature reference would materialise an $n_{\mathrm{train}} \times N_{\mathrm{ref}}$ cost matrix on the GPU, causing the memory measurement to reflect evaluation overhead rather than the cost of the method itself.

Peak memory is measured in isolated worker processes. JAX exposes its device-memory statistics as a process-monotone counter, with no in-process reset of the recorded peak: once a process has allocated memory at any point in its lifetime, the reported peak stays at that high-water mark even after the work that triggered it has been freed. Running multiple cells in the same process therefore contaminates later peaks with earlier transient allocations. To avoid this, the scaling orchestrator launches a fresh subprocess for each combination of benchmark, method, and training size, so every measurement starts from a peak of zero. Each subprocess runs all seeds for its cell and records a single cell-level peak-memory value. 

In the plots, peak memory is deduplicated to one row per cell before aggregation, so the seed rows that share a single subprocess are not treated as independent memory measurements. Before stopping the fit-time timer, the worker also forces JAX to finish all asynchronous device work — it walks the pytree of returned arrays and blocks on each leaf — so that the recorded runtime reflects actual completion rather than the moment the dispatch queue was filled. The reported scaling time is therefore an end-to-end method-call time for the worker process, and includes JIT compilation on the first seed plus warm-cache execution on later seeds.

\subsection{Implementation Details for Real-World Graph-Supported Generation on Manhattan Uber Pickups}
\label{app:impl-q3-uber}

This subsection gives implementation details for the real-data experiment on
Uber NYC pickups.  Unlike in the other two experiments, where the target law is synthetic and
analytically specified, in this experiment uses a large empirical target measure supported on
snapped pickup locations on a Manhattan road graph.  The purpose of the
experiment is to test whether the proposed neural graph-transport pipeline can
scale to a million-sample empirical target distribution while producing
realistic graph-supported target measures.

\paragraph{Data Source and Preprocessing.}
We use the 2014 Uber NYC pickup records from April through September.  Each
pickup is represented by a timestamp and a latitude--longitude coordinate.  We
restrict the spatial domain to Manhattan below 110th Street.  The retained
region is represented as a polygon in longitude--latitude coordinates, and
pickups outside this polygon are discarded before graph snapping.

The road graph is obtained from OpenStreetMap using the drivable road network inside the same Manhattan region.  The graph
is projected to a metric coordinate system in metres, simplified to an
undirected metric graph, and pruned to its largest connected component.
Parallel road geometries are deduplicated by retaining the shortest projected
polyline between each undirected node pair, and degree-one leaves are
iteratively removed.  Each retained road edge stores its projected polyline
geometry and metric length in metres.

Pickups are snapped to the nearest retained road polyline using a spatial index.
For a pickup assigned to edge \(e=(u,v)\), we record both the edge index and
the normalised arc-length coordinate
\[
    \alpha \in [0,1],
\]
where \(\alpha=0\) corresponds to \(u\) and \(\alpha=1\) corresponds to \(v\)
in the stored edge orientation.  Pickups with snap distance larger than
\(50\) metres are discarded.  In the main data
split, this preprocessing yields
\(2{,}267{,}487\) training pickups, \(361{,}466\) validation pickups, and
\(690{,}939\) test pickups after polygon clipping and snap filtering, for
\(3{,}319{,}892\) usable pickups over \(183\) calendar days.

\paragraph{Bridge Augmentation and Graph Representation.}
The Manhattan graph is bridge-augmented during preprocessing.  This produces
the metric graph used by all downstream methods, and the methods considered are configured
not to apply a second internal bridge augmentation.  The augmented graph used
in the main run has approximately \(3.0\mathrm{k}\) nodes, \(5.4\mathrm{k}\)
edges, and first Betti number \(g=2416\).  This genus is far beyond the range
where exact closest-vector enumeration on the Jacobian torus is feasible.

Although the metric graph contains virtual bridge-augmentation edges, density
visualisations and edgewise histogram metrics are reported on the retained
real road polylines.  Points landing on virtual edges are collapsed back to
their corresponding real road edge before edgewise binning.

\paragraph{Train/Validation/Test Split.}
The main experiment split is a random day-level split.  Calendar days, not individual
pickups, are partitioned into train, validation, and test sets with fractions
\(70\%/10\%/20\%\).  Thus all pickups from the same calendar day belong to a
single split, avoiding leakage of same-day demand patterns between training
and testing.

\paragraph{Source and Target Measures.}
The empirical target measure \(\nu\) is supported on snapped Uber pickup
graph points from the training split.  The source measure \(\mu\) is a
synthetic length-proportional uniform measure on the real road edges.  For
the million-run experiment, we use
\[
    N_{\mathrm{target}} = 10^6,
    \qquad
    N_{\mathrm{source}} = 10^6.
\]
The target support is subsampled without replacement from the training
pickups, and the source support is sampled uniformly along real road edges.

\paragraph{Decoder Support and Generated Measure.}
Training support size and decoder support size are kept separate.  This is
important at million-sample scale: using all \(10^6\) target points as the
post-training Gibbs support would make the conditional transport matrix
prohibitively large.  After training, the learned entropic transport kernel
is evaluated on a bounded decoder support
\[
    K_{\mathrm{dec}} = 131{,}072
\]
sampled from the target training support.  The generated measure is formed
from
\[
    M = 131{,}072
\]
held-out source query points sampled uniformly along the graph.  For a source
query \(x_i\) and decoder atom \(y_j\), the learned conditional Gibbs weights
have the form
\[
    P_\theta(y_j\mid x_i)
    \propto
    \exp\!\left(
        \frac{\psi_\theta(y_j)-c(x_i,y_j)}{\varepsilon}
    \right),
    \qquad
    y_j \in Y_{\mathrm{dec}}.
\]
The final predicted measure \(\widehat\nu_\theta\) is the graph-supported
empirical measure obtained by decoding the learned conditional distribution
at the \(M\) source query points.

\paragraph{Method Roster at Million Scale.}
The million-scale roster is Neural Euclidean (with Log Embedding),  Neural Euclidean(with Gromov Embedding), Ambient pushforward and Node Interpolation baselines
plus two reference rows: \emph{Eval.\ noise floor} (test-test) and
\emph{uniform-road} (length-proportional uniform measure).  

The two source-aware non-neural baselines (Ambient pushforward,
Node Interpolation) cannot be run at the full \(10^6\)-point training
support without OOM, because their dense \(N\times N\) ambient Sinkhorn cost
matrix would require tens of gigabytes.  We therefore subsample the
\((\mu,\nu,\mu_{\mathrm{eval}})\) tuples down to a baseline support cap
\[
    N_{\mathrm{baseline}} = 16{,}384
\]
matching \(N_{\mathrm{W}}\) below; neural methods continue to see the full
\(10^6\)-point training supports.

\paragraph{Neural Models.}
The neural Euclidean baseline represents graph points by their projected
polyline coordinates in \(\mathbb R^2\).  The input position function is
polyline-aware: snapped graph points are mapped to the actual point on the
retained road polyline rather than to a chord interpolation between graph
vertices.  The Euclidean model therefore uses the correct Manhattan road
geometry at both input and output projection stages.

Unless otherwise stated, neural potentials are multilayer perceptrons with
SiLU activations and scalar output, hidden dimensions \((512,512,512,512)\).
The two Euclidean variants share the same architecture and training schedule
but differ in their input embedding:
Neural Euclidean (Log) uses a logarithmic embedding in \(\mathbb R^2\)
(centering at a base point) with LayerNorm disabled, matching the controlled ablation protocol in the other experiments; Neural Euclidean (Gromov)  uses a
graph-distance landmark embedding with LayerNorm enabled, which is the
preset that wins on the Manhattan benchmark.

\paragraph{Regularisation and Optimisation.}
The entropic OT temperature is selected adaptively for each method unless an
explicit value is provided.  Specifically, if \(\varepsilon\) is not fixed,
we set
\[
    \varepsilon
    =
    \lambda \cdot \operatorname{median}(C),
\]
where \(C\) is a method-specific training cost subsample.  We use distinct
multipliers for the neural training Sinkhorn and the source-aware baseline
training Sinkhorn:
\[
    \lambda_{\mathrm{neural}}=10^{-2},
    \qquad
    \lambda_{\mathrm{baseline}}=0.30.
\]
The baseline value \(\lambda_{\mathrm{baseline}}=0.30\) is justified by an
\(\varepsilon\)-sweep on the source-aware baselines at million scale.  This mirrors
the adaptive temperature scheme used in the synthetic experiments and avoids
using a fixed temperature across graphs with different metric scales.

For million-scale runs, the neural optimizer uses minibatches of size
\(B=1024\).  The production preset trains for \(T=3000\) steps with cosine
learning-rate decay from initial rate \(3 \cdot 10^{-3}\), corresponding to
\(B \cdot T = 3{,}072{,}000\) target draws \(\approx 3.07\) full passes over
the \(10^6\)-point empirical target support.  

\paragraph{Evaluation Reference.}
The held-out reference measure is built from snapped test pickups.  The main
million-scale run uses
\[
    N_{\mathrm{eval}} = 131{,}072
\]
test pickups as the reference support.  Since the dense Sinkhorn cost matrix
for \(131{,}072\times131{,}072\) atoms would require tens of gigabytes of
memory, Wasserstein metrics are computed on a fixed subsample of size
\[
    N_{\mathrm{W}} = 16{,}384.
\]

\paragraph{Metrics.}
Q3 reports intrinsic graph Wasserstein metrics, edge-density metrics, and
optional spatial coverage diagnostics.  The Wasserstein metrics are computed
using the same graph-supported entropic evaluation pipeline as in the first experiment:
\[
    W_1(\widehat\nu_\theta,\nu_{\mathrm{test}}),
    \qquad
    W_2(\widehat\nu_\theta,\nu_{\mathrm{test}}).
\]
The evaluation Sinkhorn regularisation is adaptive,
\[
    \varepsilon_{\mathrm{eval}}
    =
    \lambda_{\mathrm{eval}}
    \cdot
    \operatorname{median}(D_{\mathrm{eval}}),
    \qquad
    \lambda_{\mathrm{eval}}=10^{-3},
\]
unless explicitly overridden.

Density-style metrics are computed by binning mass along the real road
polylines.  Edgewise density \(L^1\) error compares the generated and
held-out test densities on a common edgewise binning, and edgewise CDF
\(L^1\) compares cumulative mass along each edge.  The bin count is derived
from a target bin length in metres, using the median retained road-edge
length to select an effective bins-per-edge value.

\paragraph{Reference Rows: Noise Floor and Uniform-Road Baseline.}
We report two non-trainable reference rows alongside the trained methods.
The \emph{eval.\ noise floor} is a test--test reference: we draw two
independent subsamples of size
\[
    N_{\mathrm{noise}} = 32{,}768
\]
from the held-out test pickups and evaluate the same graph metrics between
them under the identical \(N_{\mathrm{W}}=16{,}384\)-point Sinkhorn budget.
This estimates the irreducible discrepancy from finite test-set sampling
and lower-bounds the score any method can attain at the chosen evaluation
resolution.  The \emph{uniform-road} reference is the length-proportional
uniform measure on the real road edges; it represents an uninformed
baseline that ignores demand structure.  Both reference rows use the same
adaptive evaluation regularisation as the trained methods.

\clearpage
\newpage

\section{Additional Experimental Results}
\label{sec-additional-results}

\subsection{Additional Results for Metric Graph Geometry Benchmarks}
\begin{table}[h]
\centering
\small
\caption{Full absolute accuracy results. Each entry reports mean $\pm$ SEM over five random seeds. Lower is better. Best result per benchmark/metric is in \textbf{bold}; second-best is \underline{underlined}.}
\label{tab:q1-full}
\setlength{\tabcolsep}{4pt}
\begin{tabular}{llcccc}
\toprule
Benchmark & Method & Graph $W_1$ & Graph $W_2$ & Density $L_1$ & Edgewise CDF $L_1$ \\
\midrule
\multirow{5}{*}{Theta} & Ambient pushforward & 0.0466 {\scriptsize $\pm$ 0.0004} & 0.0625 {\scriptsize $\pm$ 0.0005} & 0.8322 {\scriptsize $\pm$ 0.0062} & 0.0466 {\scriptsize $\pm$ 0.0004} \\
 & Node interpolation & 0.0465 {\scriptsize $\pm$ 0.0004} & 0.0624 {\scriptsize $\pm$ 0.0004} & 0.8365 {\scriptsize $\pm$ 0.0088} & 0.0465 {\scriptsize $\pm$ 0.0004} \\
 & Neural Eucl. (Log)& \underline{0.0037 {\scriptsize $\pm$ 0.0003}} & \underline{0.0231 {\scriptsize $\pm$ 0.0001}} & \underline{0.3254 {\scriptsize $\pm$ 0.0056}} & \underline{0.0053 {\scriptsize $\pm$ 0.0005}} \\
 & Neural Eucl. (Gromov) & \textbf{0.0037 {\scriptsize $\pm$ 0.0003}} & \textbf{0.0231 {\scriptsize $\pm$ 0.0001}} & 0.3256 {\scriptsize $\pm$ 0.0055} & \textbf{0.0053 {\scriptsize $\pm$ 0.0004}} \\
 & Neural Tropical & 0.0043 {\scriptsize $\pm$ 0.0002} & 0.0234 {\scriptsize $\pm$ 0.0003} & \textbf{0.3175 {\scriptsize $\pm$ 0.0079}} & 0.0058 {\scriptsize $\pm$ 0.0008} \\
\midrule
\multirow{5}{*}{Grid} & Ambient pushforward & 0.0491 {\scriptsize $\pm$ 0.0007} & 0.0660 {\scriptsize $\pm$ 0.0010} & 0.5655 {\scriptsize $\pm$ 0.0102} & 0.0491 {\scriptsize $\pm$ 0.0007} \\
 & Node interpolation & 0.0491 {\scriptsize $\pm$ 0.0003} & 0.0659 {\scriptsize $\pm$ 0.0006} & 0.5590 {\scriptsize $\pm$ 0.0074} & 0.0491 {\scriptsize $\pm$ 0.0003} \\
 & Neural Eucl. (Log) & \underline{0.0045 {\scriptsize $\pm$ 0.0002}} & \underline{0.0241 {\scriptsize $\pm$ 0.0004}} & \textbf{0.3125 {\scriptsize $\pm$ 0.0149}} & \underline{0.0074 {\scriptsize $\pm$ 0.0012}} \\
 & Neural Eucl. (Gromov) & \textbf{0.0044 {\scriptsize $\pm$ 0.0002}} & \textbf{0.0237 {\scriptsize $\pm$ 0.0003}} & \underline{0.3136 {\scriptsize $\pm$ 0.0124}} & \textbf{0.0063 {\scriptsize $\pm$ 0.0008}} \\
 & Neural Tropical & 0.0048 {\scriptsize $\pm$ 0.0002} & 0.0246 {\scriptsize $\pm$ 0.0005} & 0.3454 {\scriptsize $\pm$ 0.0134} & 0.0089 {\scriptsize $\pm$ 0.0012} \\
\midrule
\multirow{5}{*}{Wheel} & Ambient pushforward & 0.0152 {\scriptsize $\pm$ 0.0013} & 0.0354 {\scriptsize $\pm$ 0.0010} & 0.3447 {\scriptsize $\pm$ 0.0077} & 0.0214 {\scriptsize $\pm$ 0.0007} \\
 & Node interpolation & 0.0389 {\scriptsize $\pm$ 0.0007} & 0.0543 {\scriptsize $\pm$ 0.0009} & 0.4583 {\scriptsize $\pm$ 0.0139} & 0.0389 {\scriptsize $\pm$ 0.0007} \\
 & Neural Eucl. (Log) & \textbf{0.0050 {\scriptsize $\pm$ 0.0003}} & \textbf{0.0259 {\scriptsize $\pm$ 0.0013}} & \textbf{0.3301 {\scriptsize $\pm$ 0.0139}} & \underline{0.0116 {\scriptsize $\pm$ 0.0025}} \\
 & Neural Eucl.  (Gromov) & \underline{0.0051 {\scriptsize $\pm$ 0.0003}} & 0.0263 {\scriptsize $\pm$ 0.0016} & \underline{0.3312 {\scriptsize $\pm$ 0.0131}} & 0.0120 {\scriptsize $\pm$ 0.0029} \\
 & Neural Tropical & 0.0055 {\scriptsize $\pm$ 0.0003} & \underline{0.0261 {\scriptsize $\pm$ 0.0015}} & 0.3577 {\scriptsize $\pm$ 0.0136} & \textbf{0.0114 {\scriptsize $\pm$ 0.0024}} \\
\midrule
\multirow{5}{*}{Road} & Ambient pushforward & 0.2232 {\scriptsize $\pm$ 0.0024} & 0.2743 {\scriptsize $\pm$ 0.0031} & 0.7757 {\scriptsize $\pm$ 0.0042} & 0.2241 {\scriptsize $\pm$ 0.0020} \\
 & Node interpolation & 0.1808 {\scriptsize $\pm$ 0.0049} & 0.2434 {\scriptsize $\pm$ 0.0027} & 0.6838 {\scriptsize $\pm$ 0.0122} & 0.1995 {\scriptsize $\pm$ 0.0022} \\
 & Neural Eucl. (Log) & \underline{0.0110 {\scriptsize $\pm$ 0.0005}} & \textbf{0.0371 {\scriptsize $\pm$ 0.0023}} & \textbf{0.3938 {\scriptsize $\pm$ 0.0141}} & \textbf{0.0381 {\scriptsize $\pm$ 0.0051}} \\
 & Neural Eucl. (Gromov) & \textbf{0.0109 {\scriptsize $\pm$ 0.0006}} & \underline{0.0380 {\scriptsize $\pm$ 0.0023}} & \underline{0.3994 {\scriptsize $\pm$ 0.0122}} & \underline{0.0393 {\scriptsize $\pm$ 0.0051}} \\
 & Neural Tropical & 0.0138 {\scriptsize $\pm$ 0.0020} & 0.0439 {\scriptsize $\pm$ 0.0044} & 0.4132 {\scriptsize $\pm$ 0.0169} & 0.0417 {\scriptsize $\pm$ 0.0030} \\
\bottomrule
\end{tabular}
\end{table}

\subsection{Additional Results for Scalability Analysis}

\begin{figure}[htbp]
    \centering
    \includegraphics[width=0.9\linewidth]{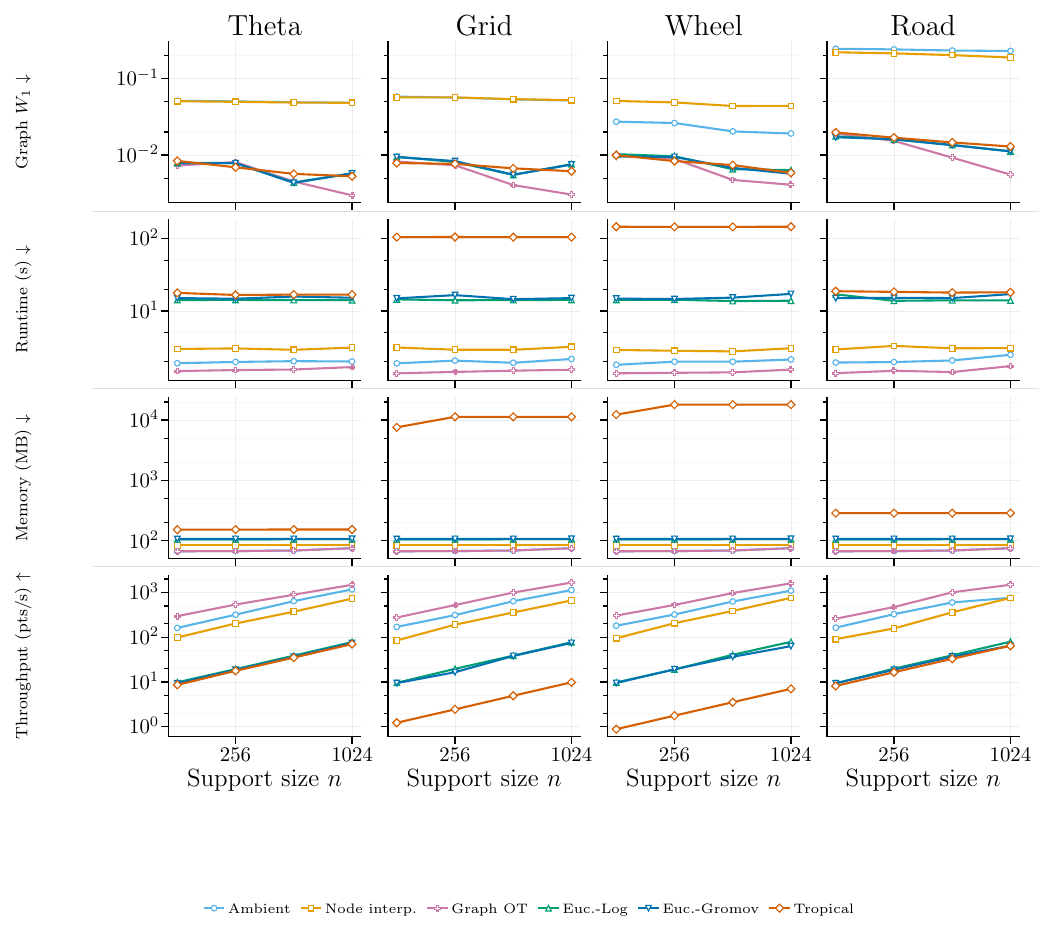}
    \caption{Rows: graph $W_1$ against the analytical target (lower is better),
  fit-time wall clock (lower is better), peak GPU memory (lower is
  better), and inference throughput $n_\mathrm{train}/\texttt{fit\_time}$
  (higher is better). Columns: the four Q2 graph families ordered by
  first Betti number, $g=2$ (Theta), $g=4$ (Grid), $g=5$ (Wheel), and
  $g\geq 7$ (Road, with bridge augmentation). Each panel shows one
  curve per method; lines are the mean across $3$ seeds.}
    \label{fig:q2_4x4}
\end{figure}

\subsection{Additional Results for Real-World Graph-Supported Generation on Manhattan Uber Pickups}

\begin{table}[h]
\centering
\caption{Generation time on the Manhattan road metric graph.}
\label{tab:q3-million-time}
\begin{tabular}{lc}
\toprule
Method & Time (s) $\downarrow$ \\
\midrule
Eval.\ noise floor &  21.9 {\scriptsize $\pm$ 1.4}  \\
\midrule
Ambient pushforward & 35.0 {\scriptsize $\pm$ 0.2} \\
Node interpolation & 136.8 {\scriptsize $\pm$ 1.7} \\
Neural Euclidean (Log) & 105.9 {\scriptsize $\pm$ 0.9} \\
Neural Euclidean (Gromov) & 112.4 {\scriptsize $\pm$ 2.2} \\
\bottomrule
\end{tabular}
\end{table}

\clearpage
\newpage

\section{Proofs}

\subsection{Proof of Theorem~\ref{thm:graph-supported-recovery}}
\label{app-proof-graph-supported-recovery}
\begin{proof}
Fix the embedding $\Psi:\Gamma\to\mathcal M_\Psi$. Since $\Gamma$ is compact
and $\mathcal M_\Psi$ is Hausdorff, the continuous injective map
$\Psi:\Gamma\to\mathcal M_\Psi$ is a homeomorphism onto its image
$\Psi(\Gamma)$. Hence
\[
    \Psi^{-1}:\Psi(\Gamma)\to\Gamma
\]
is well-defined and continuous.

By the assumptions that $\mathcal F$ is dense in
$C(\mathbb R^n,\mathbb R)$ under the ucc topology and that
$\varphi_\Psi:K_{\mathbb Q_\Psi}\to\mathbb R^n$ is continuous and injective,
the recovery result of \cite[Theorem 4.1]{micheli_entropic}, applied with
\[
(\mathcal M,\mu,\nu,c)
=
(\mathcal M_\Psi,\mathbb P_\Psi,\mathbb Q_\Psi,c_\Psi),
\]
gives a sequence
\[
g_m\in\mathsf C_{\mathbb Q_\Psi}(\varphi_\Psi^*\mathcal F)
\]
whose induced embedded Gibbs laws satisfy
\[
\|\pi_{m,\Psi}^{\varepsilon}
-
\pi_{\varepsilon,\Psi}^{\star}\|_{\mathrm{TV}}
\longrightarrow 0 .
\]

Let $\Pi_{m,\Psi,t}^{\varepsilon}$ denote the embedded heat-smoothed joint law
before projection:
\[
\Pi_{m,\Psi,t}^{\varepsilon}(du,dz)
=
\mathbb P_\Psi(du)\,
Q_{m,\Psi,t}^{\varepsilon}(u,dz),
\]
where $Q_{m,\Psi,t}^{\varepsilon}$ is obtained by heat-smoothing the embedded
conditional induced by $g_m$. Let $\Pi_{\varepsilon,\Psi,t}^{\star}$ denote the
corresponding heat-smoothed version of the embedded entropic optimizer
$\pi_{\varepsilon,\Psi}^{\star}$. Since heat smoothing is a Markov kernel acting
on the second coordinate, it is total-variation contractive. Hence, for every
$t>0$,
\[
\|\Pi_{m,\Psi,t}^{\varepsilon}
-
\Pi_{\varepsilon,\Psi,t}^{\star}\|_{\mathrm{TV}}
\le
\|\pi_{m,\Psi}^{\varepsilon}
-
\pi_{\varepsilon,\Psi}^{\star}\|_{\mathrm{TV}}
\longrightarrow 0 .
\]
For the two ambient geometries considered in this paper, namely Euclidean space
and the flat torus, the heat kernel is an approximate identity. Therefore
\[
\Pi_{\varepsilon,\Psi,t}^{\star}
\rightharpoonup
\pi_{\varepsilon,\Psi}^{\star}
\qquad
(t\downarrow0).
\]
Consequently, for any sequence $t_m\downarrow0$,
\begin{equation}
\label{eq:embedded-heat-recovery-proof}
\Pi_{m,\Psi,t_m}^{\varepsilon}
\rightharpoonup
\pi_{\varepsilon,\Psi}^{\star}.
\end{equation}

It remains to pass this convergence through the projection--pullback map. Recall
that
\[
\rho_\Psi^\Gamma
=
\Psi^{-1}\circ \operatorname{proj}_{\Psi(\Gamma)}
:
\mathcal M_\Psi\to\Gamma,
\]
where $\operatorname{proj}_{\Psi(\Gamma)}$ is a fixed measurable nearest-point
projection onto $\Psi(\Gamma)$. Define
\[
G_\Psi:\mathcal M_\Psi\times\mathcal M_\Psi\to\Gamma\times\Gamma,
\qquad
G_\Psi(u,z)
=
\bigl(\rho_\Psi^\Gamma(u),\rho_\Psi^\Gamma(z)\bigr).
\]
The graph-supported law generated by the algorithm is
\[
\pi_{m,\Gamma,t_m}^{\varepsilon,\Psi}
=
(G_\Psi)_\#\Pi_{m,\Psi,t_m}^{\varepsilon}.
\]
Indeed, disintegrating
\[
\Pi_{m,\Psi,t_m}^{\varepsilon}(du,dz)
=
\mathbb P_\Psi(du)Q_{m,\Psi,t_m}^{\varepsilon}(u,dz)
\]
and using $\rho_\Psi^\Gamma(\Psi(x))=x$ gives
\[
\pi_{m,\Gamma,t_m}^{\varepsilon,\Psi}(dx,dy)
=
\mathbb P_\Gamma(dx)\,
(\rho_\Psi^\Gamma)_\#
Q_{m,\Psi,t_m}^{\varepsilon}(\Psi(x),\cdot)(dy),
\]
which is the graph-supported law defined by the projection--pullback
procedure.

We next verify that the continuous mapping theorem applies to the measurable
map $G_\Psi$. Since
$\pi_{\varepsilon,\Psi}^{\star}\in
\Pi(\mathbb P_\Psi,\mathbb Q_\Psi)$ and both marginals are supported on
$\Psi(\Gamma)$,
\[
\pi_{\varepsilon,\Psi}^{\star}
\bigl(\Psi(\Gamma)\times\Psi(\Gamma)\bigr)=1.
\]
Thus it suffices to show that $G_\Psi$ is continuous at every point of
$\Psi(\Gamma)\times\Psi(\Gamma)$.

Let $v\in\Psi(\Gamma)$ and let $z_n\to v$ in $\mathcal M_\Psi$. Set
\[
s_n:=\operatorname{proj}_{\Psi(\Gamma)}(z_n).
\]
Since $v\in\Psi(\Gamma)$ is an admissible comparison point, nearest-point
optimality gives
\[
d_\Psi(s_n,z_n)\le d_\Psi(v,z_n).
\]
Hence
\[
d_\Psi(s_n,v)
\le
d_\Psi(s_n,z_n)+d_\Psi(z_n,v)
\le
2d_\Psi(z_n,v)
\to0.
\]
Therefore
\[
\operatorname{proj}_{\Psi(\Gamma)}(z_n)\to v.
\]
Since $\Psi^{-1}$ is continuous on $\Psi(\Gamma)$, it follows that
\[
\rho_\Psi^\Gamma(z_n)
=
\Psi^{-1}\bigl(\operatorname{proj}_{\Psi(\Gamma)}(z_n)\bigr)
\to
\Psi^{-1}(v).
\]
Thus $\rho_\Psi^\Gamma$ is continuous at every point of $\Psi(\Gamma)$.
Consequently, $G_\Psi$ is continuous at every point of
$\Psi(\Gamma)\times\Psi(\Gamma)$. If $D_{G_\Psi}$ denotes the discontinuity set
of $G_\Psi$, then
\[
\pi_{\varepsilon,\Psi}^{\star}(D_{G_\Psi})=0.
\]

By the mapping theorem for weak convergence applied to
\eqref{eq:embedded-heat-recovery-proof},
\[
(G_\Psi)_\#\Pi_{m,\Psi,t_m}^{\varepsilon}
\rightharpoonup
(G_\Psi)_\#\pi_{\varepsilon,\Psi}^{\star}.
\]
The left-hand side is
$\pi_{m,\Gamma,t_m}^{\varepsilon,\Psi}$. On the support of
$\pi_{\varepsilon,\Psi}^{\star}$, both coordinates lie in $\Psi(\Gamma)$, so
\[
G_\Psi(u,v)
=
(\Psi^{-1}(u),\Psi^{-1}(v)).
\]
Therefore
\[
(G_\Psi)_\#\pi_{\varepsilon,\Psi}^{\star}
=
(\Psi^{-1},\Psi^{-1})_\#
\pi_{\varepsilon,\Psi}^{\star}
=
\pi_{\varepsilon,\Gamma}^{\star,\Psi}.
\]
Thus
\[
\pi_{m,\Gamma,t_m}^{\varepsilon,\Psi}
\rightharpoonup
\pi_{\varepsilon,\Gamma}^{\star,\Psi}.
\]
Taking, for example, $m_k=k$ and any sequence $t_k\downarrow0$ gives
\[
\pi_{m_k,\Gamma,t_k}^{\varepsilon,\Psi}
\rightharpoonup
\pi_{\varepsilon,\Gamma}^{\star,\Psi}.
\]

Finally, since
$\pi_{\varepsilon,\Psi}^{\star}\in
\Pi(\mathbb P_\Psi,\mathbb Q_\Psi)$ with
$\mathbb P_\Psi=\Psi_\#\mathbb P_\Gamma$ and
$\mathbb Q_\Psi=\Psi_\#\mathbb Q_\Gamma$, the pullback
\[
\pi_{\varepsilon,\Gamma}^{\star,\Psi}
=
(\Psi^{-1},\Psi^{-1})_\#
\pi_{\varepsilon,\Psi}^{\star}
\]
has marginals $\mathbb P_\Gamma$ and $\mathbb Q_\Gamma$. Hence
\[
\pi_{\varepsilon,\Gamma}^{\star,\Psi}
\in
\Pi(\mathbb P_\Gamma,\mathbb Q_\Gamma).
\]
This proves the claim.
\end{proof}

\subsection{Proof of Lemma~\ref{lem:augmented-bridgeless}}
\label{app-proof-lemma-augmented-bridgeless}
\begin{proof}
Recall that an edge \(e\) is a bridge of a connected metric graph if and only if
\(\widetilde\Gamma \setminus e^\circ\) is disconnected. We show that this never happens in
\(\widetilde\Gamma\).

Let \(e\) be any edge of \(\widetilde\Gamma\). We distinguish three cases.

\smallskip
\noindent\emph{(i) \(e=\bar b\) is a virtual edge.}
By construction, \(\bar b\) connects the same endpoints \(u_b,v_b\) as the corresponding bridge
\(b\in\mathfrak B(\Gamma)\). Hence \(b\) and \(\bar b\) give two distinct \(u_b\)--\(v_b\) paths in
\(\widetilde\Gamma\), and therefore \(\bar b\) lies on the (simple) cycle \(b\cup \bar b\).
In particular, removing \(\bar b^\circ\) does not separate \(u_b\) from \(v_b\) (they remain connected
through \(b\)), so \(\widetilde\Gamma\setminus \bar b^\circ\) is connected.

\smallskip
\noindent\emph{(ii) \(e=b\) is an original bridge of \(\Gamma\).}
As in (i), the union \(b\cup \bar b\) is a cycle in \(\widetilde\Gamma\). Thus \(u_b\) and \(v_b\)
remain connected after removing \(b^\circ\) (via \(\bar b\)), and the graph cannot disconnect:
\(\widetilde\Gamma\setminus b^\circ\) is connected.

\smallskip
\noindent\emph{(iii) \(e\) is an original non-bridge edge of \(\Gamma\).}
Since \(e\) is not a bridge in \(\Gamma\), the space \(\Gamma\setminus e^\circ\) is connected.
Moreover, \(\Gamma\setminus e^\circ\) embeds naturally as a subspace of \(\widetilde\Gamma\setminus e^\circ\),
because the augmentation only adds extra edges and does not remove any existing ones.
A connected subspace cannot become disconnected by adding more edges/paths, hence
\(\widetilde\Gamma\setminus e^\circ\) is connected.

\smallskip
In all cases, \(\widetilde\Gamma\setminus e^\circ\) is connected, so \(\widetilde\Gamma\) has no bridges.
\end{proof}

\subsection{Proof of Lemma~\ref{lem:isometric-inclusion}}
\label{app-proof-lem:isometric-inclusion}
\begin{proof}
Because \(\Gamma\) is a subgraph of \(\widetilde\Gamma\), every path in \(\Gamma\) is also a path in
\(\widetilde\Gamma\). Taking infima over path lengths immediately gives
\[
d_{\widetilde\Gamma}(\iota(x),\iota(y)) \;\le\; d_\Gamma(x,y).
\]

We prove the reverse inequality by showing that any \(\widetilde\Gamma\)-path between \(\iota(x)\) and
\(\iota(y)\) can be modified into a \(\Gamma\)-path of no greater length.  Let \(\gamma\) be a
rectifiable path in \(\widetilde\Gamma\) from \(\iota(x)\) to \(\iota(y)\). Consider any connected
sub-arc of \(\gamma\) contained in a virtual edge \(\bar b\). Its endpoints are two points
\(p,q\in \bar b\) (possibly equal, possibly vertices, and not necessarily the endpoints of \(\bar b\)).
Using the fixed arc-length parametrizations, identify \(\bar b\) with \([0,\ell(\bar b)]\) and \(b\)
with \([0,\ell(b)]\) so that corresponding endpoints match. Since \(\ell(\bar b)\ge \ell(b)\), setting $\delta=\ell(\bar{b})-\ell(b)\ge 0$, the
map
\[
T_b:[0,\ell(\bar b)]\to[0,\ell(b)],\qquad T_b(t):=\min\{t,\ell(b)\},
\]
is \(1\)-Lipschitz and preserves the common endpoint at \(0\) (and collapses only the extra tail if
\(\delta>0\)). Pushing the segment of \(\gamma\) in \(\bar b\) forward by \(T_b\) yields a path in \(b\)
connecting the corresponding points in \(b\) whose length is \emph{no larger} than the original
segment in \(\bar b\). Replacing each maximal sub-arc of \(\gamma\) lying in any virtual edge in this
way produces a new path \(\gamma'\subset\Gamma\) from \(x\) to \(y\) with
\[
\mathrm{len}(\gamma') \;\le\; \mathrm{len}(\gamma).
\]

Since \(\gamma\) was arbitrary, taking the infimum over all \(\widetilde\Gamma\)-paths gives
\[
d_\Gamma(x,y)\;\le\; d_{\widetilde\Gamma}(\iota(x),\iota(y)).
\]
Together with the first inequality this proves
\(d_{\widetilde\Gamma}(\iota(x),\iota(y))=d_\Gamma(x,y)\), i.e.\ \(\iota\) is an isometric embedding.
\end{proof}

\subsection{Proof of Lemma~\ref{lem:torus-support-physical}}
\label{app-proof-lem:torus-support-physical}
\begin{proof}
By construction, \(\widetilde{\mathbb P}_\Gamma=\iota_\#\mathbb P_\Gamma\) and
\(\widetilde{\mathbb Q}_\Gamma=\iota_\#\mathbb Q_\Gamma\), hence
\(\mathrm{supp}(\widetilde{\mathbb P}_\Gamma)\subset \iota(\Gamma)\) and
\(\mathrm{supp}(\widetilde{\mathbb Q}_\Gamma)\subset \iota(\Gamma)\).
Equivalently, for every Borel set \(A\subset\widetilde\Gamma\) disjoint from \(\iota(\Gamma)\) we have
\(\widetilde{\mathbb P}_\Gamma(A)=\widetilde{\mathbb Q}_\Gamma(A)=0\).

Let \(B\subset\mathbb T^{\tilde g}\) be a Borel set disjoint from \(\widetilde\Phi_p(\iota(\Gamma))\).
Then \((\widetilde\Phi_p)^{-1}(B)\cap \iota(\Gamma)=\varnothing\), so
\[
\widetilde{\mathbb P}_{\mathfrak{J}}(B)
=\widetilde{\mathbb P}_\Gamma\bigl((\widetilde\Phi_p)^{-1}(B)\bigr)=0,
\qquad
\widetilde{\mathbb Q}_{\mathfrak{J}}(B)
=\widetilde{\mathbb Q}_\Gamma\bigl((\widetilde\Phi_p)^{-1}(B)\bigr)=0.
\]
Thus both pushforward measures assign zero mass to every Borel set disjoint from
\(\widetilde\Phi_p(\Gamma)\), i.e.\ they are supported on \(\widetilde\Phi_p(\Gamma)\).
\end{proof}

\subsection{Proof of Lemma~\ref{lem:torus-plan-supported}}
\label{app-proof-lem:torus-plan-supported}

\begin{proof}
Set \(S:=\widetilde\Phi_p(\Gamma)\) and \(A:=\mathbb T^{\tilde g}\setminus S\). Let
\(\pi\in\Pi(\widetilde{\mathbb P}_{\mathfrak{J}},\widetilde{\mathbb Q}_{\mathfrak{J}})\). Since the
first marginal of \(\pi\) is \(\widetilde{\mathbb P}_{\mathfrak{J}}\), we have
\[
\pi(A\times \mathbb T^{\tilde g})
=\widetilde{\mathbb P}_{\mathfrak{J}}(A)=0,
\]
and similarly, because the second marginal is \(\widetilde{\mathbb Q}_{\mathfrak{J}}\),
\[
\pi(\mathbb T^{\tilde g}\times A)
=\widetilde{\mathbb Q}_{\mathfrak{J}}(A)=0.
\]
Therefore
\[
\pi\Bigl(\bigl(\mathbb T^{\tilde g}\times\mathbb T^{\tilde g}\bigr)\setminus(S\times S)\Bigr)
\le \pi(A\times \mathbb T^{\tilde g})+\pi(\mathbb T^{\tilde g}\times A)=0,
\]
which proves that \(\pi\) is concentrated on \(S\times S\).
\end{proof}

\subsection{Proof of Proposition~\ref{prop:OT-compat-bridge}}
\label{app-proof-prop:OT-compat-bridge}
\begin{proof}
Existence of $\widetilde\pi_\Gamma^\star$ satisfying \eqref{eq:pushforward-opt} is exactly the bridgeless result
applied to $\widetilde\Gamma$. The zero-virtual-mass property follows because both marginals are supported on
$\iota(\Gamma)$, so any coupling between them is supported on $\iota(\Gamma)\times\iota(\Gamma)$.
Finally, identifying $\iota(\Gamma)$ with $\Gamma$ gives the claimed coupling on the original graph.
\end{proof}

%\newpage
%\input{checklist.tex}

\end{document}